\documentclass{article} % For LaTeX2e
\usepackage{iclr2026_conference,times}

\usepackage{hyperref}
\usepackage{url}
\usepackage{fvextra}

\usepackage{microtype}
\usepackage{hyperref}
\usepackage{url}
\usepackage{booktabs}
\usepackage{multirow}
\usepackage{graphicx}
\usepackage{amsmath,amssymb}
\usepackage[table]{xcolor}
\usepackage{caption}
\usepackage{pifont}
\usepackage{titletoc}

\newcommand{\cmark}{\ding{51}}
\newcommand{\xmark}{\ding{55}}

\usepackage{algorithm}
\usepackage{algpseudocode}
\usepackage{subcaption}  % preamble
\usepackage{wrapfig}
\usepackage{tcolorbox}
\tcbuselibrary{skins, breakable}

\usepackage{amsthm}
\newtheorem{theorem}{Theorem}[section]
\newtheorem{proposition}[theorem]{Proposition}
\newtheorem{lemma}[theorem]{Lemma}
\newtheorem{assumption}[theorem]{Assumption}

\definecolor{mygreen}{rgb}{0.27,0.70,0.23}   % DEF1D3
\definecolor{myyellow}{rgb}{0.74,0.71,0.02}  % FDF4D1

\definecolor{cadmiumgreen}{rgb}{0.0, 0.42, 0.24}

\definecolor{myred}{rgb}{0.7, 0.3, 0.0}
\definecolor{myblue}{rgb}{0.2, 0.3, 0.6}

\usepackage{amsmath,amsfonts,bm}

\def\eqref#1{equation~\ref{#1}}
\def\1{\bm{1}}

\DeclareMathAlphabet{\mathsfit}{\encodingdefault}{\sfdefault}{m}{sl}
\SetMathAlphabet{\mathsfit}{bold}{\encodingdefault}{\sfdefault}{bx}{n}

\usepackage{xcolor}

\definecolor{RAILShadow}{HTML}{B27C28}
\definecolor{RAILLive}{HTML}{467F75}

\colorlet{RAILShadowBG}{RAILShadow!10}
\colorlet{RAILLiveBG}{RAILLive!10}

\newcommand{\ShadowHeader}[1]{%
    \Statex
    \colorbox{RAILShadowBG}{%
        \parbox{\dimexpr\linewidth-2\fboxsep\relax}{%
            \textcolor{RAILShadow}{\strut\textbf{#1}}%
        }%
    }%
}

\newcommand{\LiveHeader}[1]{%
    \Statex
    \colorbox{RAILLiveBG}{%
        \parbox{\dimexpr\linewidth-2\fboxsep\relax}{%
            \textcolor{RAILLive}{\strut\textbf{#1}}%
        }%
    }%
}

\DefineVerbatimEnvironment{PromptVerbatim}{Verbatim}{
    fontsize=\footnotesize,
    breaklines=true,
    breakanywhere=true,
    commandchars=\\\{\}
}

\newtcolorbox{promptbox}[2][]{%
    enhanced,
    colback=#2!4!white,
    colframe=#2!70!black,
    coltitle=white,
    fonttitle=\bfseries,
    title={#1},
    arc=2mm,
    boxrule=0.8pt,
    left=2.5mm,
    right=2.5mm,
    top=1.5mm,
    bottom=1.5mm
}

\newtcolorbox{definitionbox}{
  colback=gray!8,
  colframe=black,
  boxrule=1pt,
  arc=5pt,
  left=6pt,
  right=6pt,
  top=3pt,
  bottom=3pt,
  before skip=5pt,
  after skip=5pt,
  breakable
}

\definecolor{promptteal}{HTML}{3F7F78}
\definecolor{promptblue}{HTML}{3F5F9F}
\definecolor{promptorange}{HTML}{B87528}
\definecolor{promptviolet}{HTML}{76569A}

\DefineVerbatimEnvironment{promptverb}{Verbatim}{
    fontsize=\footnotesize,
    baselinestretch=0.96,
    breaklines=true,
    breakanywhere=false,
    breakautoindent=false,
    breakindent=0pt,
    breaksymbolleft={},
    breaksymbolright={},
    tabsize=2
}

\newcommand{\promptsection}[1]{%
    \par\vspace{0.7mm}%
    \noindent\textbf{\footnotesize #1}\par
    \vspace{-0.6mm}%
}

\title{Optimizing What Policies Learn From: Recoverability \!\!-aware Rollout Intervention Learning}

\author{
    Zheyuan Zhang\textsuperscript{1}\textsuperscript{*}, 
    Manqing Mao\textsuperscript{2},
    Hong Wang\textsuperscript{2},
    Zhuoer Wang\textsuperscript{2},
    Samson Koelle\textsuperscript{2}, \\
    \textbf{Jie Yuan\textsuperscript{2},
    Yanjun Lin\textsuperscript{2},
    James Feng\textsuperscript{2},
    Nikki Lijing Kuang\textsuperscript{2},
    Yanfang Ye\textsuperscript{1},
    Wei Niu \textsuperscript{2}\textsuperscript{$\dagger$}} \\
    \textsuperscript{1}University of Notre Dame, 
    \textsuperscript{2}Amazon, Inc \\
    \textsuperscript{*}Work done during internship at Amazon.
    \textsuperscript{$\dagger$}Corresponding Author \\
    \texttt{zzhang42@nd.edu}, \texttt{niuwei@amazon.com}
}

\iclrfinalcopy % Uncomment for camera-ready version, but NOT for submission.
\begin{document}

\maketitle

\vspace{-15pt}
\begin{abstract}

Critic-free group-based RL has become a scalable paradigm for LLM post-training. However, its effectiveness is constrained by a central limitation: rollouts are allocated uniformly even though their learning value varies substantially across tasks and trajectory states. Despite recent efforts to treat rollout generation as an adaptive intervention, existing methods suffer from two core gaps: the \textbf{non-stationary gap}, which calls for intervention strategies from learning an evolving policy rather than relying on heuristic rules, and the \textbf{non-scalar gap}, which calls for structured control over \emph{where} and \emph{how} to intervene rather than merely deciding \emph{how many} rollouts to generate. To close these gaps, we introduce \textbf{Recoverability-Aware Intervention Learning (RAIL)}, a training-time framework that turns rollout generation into a learnable intervention process by optimizing structured decisions according to their realized \emph{recoverability gains}. Specifically, RAIL first casts intervention selection as an online contextual-bandit problem and then trains a recoverability controller from intervention traces collected through a shadow-to-live procedure, performing intervention learning while policy evolves. Finally, we evaluate RAIL for its \textbf{\emph{effectiveness}}, \textbf{\emph{adaptivity}}, \textbf{\emph{expressiveness}}, and \textbf{\emph{efficiency}}, showing consistent gains with constrained rollout budgets. Together, our results establish recoverability-aware intervention as a principled path toward more informative rollout generation, enabling post-training to optimize from stronger and less redundant learning signals.

\vspace{-10pt}
\end{abstract}

\section{Introduction}
\vspace{-5pt}

\begin{figure}[b]
	\centering
    \vspace{-10pt}
	\includegraphics[width=1\linewidth]{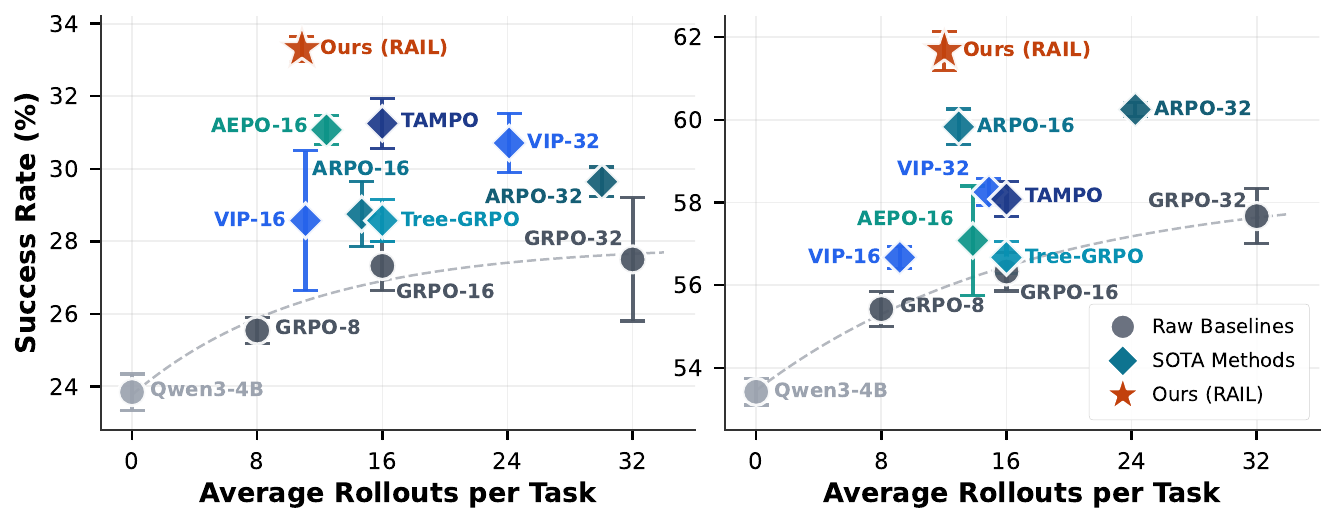}
        \vspace{-20pt}
	\caption{\textbf{Performance-rollout Results Preview on AgentBench.} RAIL achieves the highest success rates with comparable or fewer rollouts than uniform GRPO and adaptive intervention baselines on OS task (Left) and DB task(right). More results on further benchmarks and analysis in Section~\ref{sec:exp}.
    }
        
    \label{fig:opening}
\end{figure}

Despite the remarkable success of large-scale pretraining, recent LLM development increasingly relies on post-training to elicit reasoning and task-specific behaviors~\citep{tong2024dart,rafailov2023direct}. In this regime, progress depends not only on model capacity, but also on how effectively and efficiently training extracts learning signals from model-generated outputs and environment feedback. Critic-free, group-based RL methods such as GRPO~\citep{shao2024deepseekmath} improve scalability by optimizing policies from relative rewards over multiple sampled responses~\citep{liu2025understanding, ahmadian2024back}. Yet a key drawback is that they typically allocate rollout budgets uniformly, even though the available learning signal varies substantially across tasks and trajectory states in realistic settings~\citep{dong2025agentic}. This mismatch is amplified in agentic settings, where rewards are sparse, environment interactions are costly, and early decisions can make later recovery difficult~\citep{dong2026agentic}. As a result, the marginal value of additional rollouts becomes highly state-dependent: Some states benefit from additional exploration because it reveals reward contrast or alternative successful trajectories~\citep{zheng2025first, wei2026entropy}, whereas saturated or unrecoverable states only consume rollout budget, weakening optimization by mixing informative and uninformative trajectories in the same update~\citep{wang2026ragen}.

To tackle this problem, prior work has made rollout allocation adaptive from task-level allocation~\citep{yang2025depth, nguyen2026adaptive} to step-level branching~\citep{zheng2025first, dong2025agentic}, using signals such as reward variance~\citep{xiong2025reinforce, fang2026allocate} or internal signals such as entropy and uncertainty~\citep{wei2026entropy, duan2025uprop}.
% \textcolor{red}{We define this intervention utility as recoverability: a trajectory state is recoverable if targeted branching or modified exploration can make its rollout group more informative for policy optimization than default continuation. Recoverability concerns the intervention-induced improvement in the
% training signal, rather than the immediate expected return of the
% intervened rollout group.} 
However, they generally suffer from \textbf{Two Core Gaps: 1) The non-stationary gap:} existing intervention policies are largely static heuristics, which are only indirect proxies for whether the rollout intervention can expose useful learning signals, and their intervention gain can shift as policies evolve. \textbf{2) The non-scalar gap:} prior methods reduce intervention to a scalar budget decision, such as \emph{how many} rollouts to generate, which provides no principled way to coordinate coupled intervention choices, such as \emph{where} or \emph{how} to branch. Thus, a question naturally arises: \textbf{\emph{Can rollout intervention be learned online from prior realized interventions, while coordinating structured decisions about where and how to intervene?}}

To answer this question, we introduce \textbf{Recoverability-Aware Intervention Learning (RAIL)}, a training-time framework that learns structured intervention decisions from observed \emph{recoverability gains}, where \textbf{recoverability} measures how much an intervention improves the reward-distribution signal at a task or trajectory state
% {\color{red}
% where \textbf{recoverability} measures the expected improvement in the finite rollout-group reward-contrast signal induced by an intervention
% }
(Formal definition in Section~\ref{sec:prelim}). Instead of relying on fixed uncertainty or difficulty proxies, RAIL continually updates a recoverability controller from intervention outcomes and uses it to decide where to branch and select intervention strategies over a structured space. \textbf{Specifically,} RAIL first formulates recoverability learning as an online contextual-bandit problem, where the context is the current trajectory state, the action is a structured intervention, and the observed utility is the realized recoverability gain. It then learns a recoverability controller from intervention traces and deploys it through a structured shadow-to-live procedure: a short shadow phase collects supervision, while a utility-gated live phase selectively applies the controller during policy optimization. This turns rollout generation from a fixed procedure into an adaptive, outcome-driven training loop. Across various benchmarks, RAIL improves \textbf{\emph{effectiveness, adaptivity, expressiveness, and efficiency}} over SOTA intervention baselines. These results show that learning can substantially benefit structured rollout intervention, thereby strengthening policy optimization, pointing toward a broader paradigm for post-training in which both model parameters and the rollout-generation process are adapted to produce higher-value learning signals. Overall, our contributions are:
\begin{itemize}
    \item \textbf{To bridge the non-stationary gap,} we introduce \emph{recoverability} as an outcome-based principle for rollout intervention and propose an online recoverability controller that learns intervention utility from realized reward-distribution gains, rather than relying on fixed uncertainty or difficulty heuristics.

    \item \textbf{To bridge the non-scalar gap,} we formulate rollout intervention over a structured decision space and develop a shadow-to-live deployment procedure that safely collects intervention traces before applying utility-gated online branching during policy optimization.

    \item We evaluate RAIL along four complementary axes: \textbf{effectiveness, adaptivity, expressiveness, and efficiency.} Across various benchmarks, RAIL improves over SOTA heuristic intervention baselines, highlighting a more informative and principled direction for rollout intervention that provides stronger optimization signals.
\end{itemize}

\section{Rethinking Rollout Intervention through Recoverability}
\label{sec:rethink}

\subsection{Preliminaries}
\label{sec:prelim}

\textbf{GRPO. } 
Group Relative Policy Optimization (GRPO) is a critic-free policy gradient method for reasoning-oriented RL \cite{shao2024deepseekmath}. For each input $x$, it samples a group of responses $Y=\{y_i\}_{i=1}^G$ and estimates the gradient using normalized rewards:

\vspace{-25pt}
\begin{equation}
\hat{g}(x) = \frac{1}{G} \sum_{i=1}^{G} \hat{A}_i \, \nabla_\theta \log \pi_\theta(y_i \mid x), \quad
\hat{A}_i = \frac{R(y_i) - \mu_R}{\sigma_R + \epsilon}.
\end{equation}

\textbf{Uniform Sampling Limitation.} While GRPO allocates a fixed rollout budget to each task and computes advantages from uniformly sampled trajectories, the underlying optimization signals are inherently non-uniform: \emph{Some decisions carry richer signals, whereas others contribute little useful information.} This mismatch can waste training compute and dilute policy-gradient updates, motivating \emph{rollout intervention} as an adaptive mechanism for deciding where and how additional rollouts should be generated across tasks and trajectory states.

\textbf{Rollout Intervention.}
We define \emph{rollout intervention} as a controlled modification to the default rollout generation process. Given an input $x$, rather than sampling a fixed group of rollouts from the root-level, it introduces a choice $b\in\mathcal{B}$, where $\mathcal{B}$ denotes a structured intervention space, according to heuristics or learned controllers that determines how the rollout set $Y_b$ is produced. Formally, let $z$ denote the initial or intermediate state of $x$. We write $Y_b \sim Q_{\theta,b}(\cdot \mid z)$. An intervention may adjust the rollout budget, branch from selected intermediate states, or alter the exploration behavior used to generate continuations under the same policy $\pi_\theta$. In this view, rollout generation becomes an optimization object in its own right: \emph{instead of passively collecting samples, training can actively shape the rollout distribution toward trajectories that provide stronger learning signal.}

\textbf{Recoverability.}
Rollout intervention is beneficial only when additional exploration can expose richer learning signals. For tasks that are already resolved, or states whose future outcomes are overly constrained, further rollouts solely add cost while diluting the optimization signal. In this paper, we introduce \textbf{\emph{recoverability}}, defined as the intervention gain that remains available at a task or trajectory state. Formally, let $z$ denote a rollout state and $b\in\mathcal{B}$ an intervention choice. We define:
\begin{equation}
    \Delta_\theta(z,b)
    =
    \mathbb{E}\!\left[I(Y_b)-I(Y_{\varnothing}) \mid z,b,\pi_\theta\right],
\end{equation}
where $Y_b$ and $Y_{\varnothing}$ denote the finite rollout groups generated after applying intervention $b$ and under the default continuation, respectively. 
$I(Y)=\frac{1}{|Y|}\sum_{y_i\in Y}(R(y_i)-\bar{R}_Y)^2$ measures the reward contrast of a rollout group. A state is recoverable if there exists an intervention that yields positive expected gain, i.e., $\max_{b\in\mathcal{B}} \Delta_\theta(z,b) > 0$. \textbf{We instantiate the recoverability signal by finite-group reward contrast}, since in group-based policy optimization with verifiable rewards, reward variance is a principal source of non-degenerate relative-advantage signal~\citep{nguyen2026adaptive}; We provide detailed justification of why it can faithfully capture recoverability in Appendix~\ref{app:variance_signal}.
% {\color{red}
% a criterion motivated by \cite{nguyen2026adaptive} for assessing whether a rollout group has non-degenerate reward variation needed for useful relative advantages.
% }
%following the observation that in group-based policy optimization with verifiable rewards, non-degenerate reward contrast appears as a central reward-dependent factor in the relative-advantage and projected-gradient variance signals; we provide the detailed justification in Appendix~\ref{app:variance_signal}, building on the variance analysis of~\citet{nguyen2026adaptive}.

Intuitively, recoverability is high when intervention can reveal outcome variation missed by the default rollout process: for example, by expanding under-explored viable continuations, or by branching before early errors accumulate into effectively irreversible failures, analogous to absorbing states in Markov Decision Processes (MDP) \cite{sutton1998reinforcement}. This concern is amplified in agentic RL, where rewards are sparse, rollouts are costly, and environment interactions can be irreversible.

\textbf{Task Formulation.}
We frame rollout intervention as a recoverability-aware decision problem: at each rollout candidate state, choose the intervention with the largest expected recoverability gain:
% \begin{equation}
%    b^{*}(z) \in \arg\max_{b\in\mathcal{B}} \mathbb{E}[\Delta(z,b)\mid z].
% \end{equation}

\vspace{-10pt}
\begin{equation}
    b_{\theta}^{*}(z)
    \in
    \operatorname*{arg\,max}_{b \in \mathcal{B}}
    \Delta_{\theta}(z,b).
\end{equation}
\vspace{-10pt}

Since this gain is not observed before intervention, existing methods approximate it with fixed heuristics such as uncertainty and difficulty. This exposes \textbf{two core gaps}: recoverability is difficult to capture with static heuristics as policy optimization evolves, and effective intervention cannot be reduced to a fixed scalar strategy. We elaborate these two core gaps in the next subsection.

\subsection{Challenges in Recoverability-Aware Intervention}
\label{sec:challenges}

\textbf{The Non-stationary Gap.}
As discussed above, existing methods allocate additional rollouts using heuristic signals, yet these signals are merely indirect proxies for the intervention gain $\Delta_\theta(z,b)$: a state that appears uncertain may already be unrecoverable, or appears difficult and requires a particular intervention before useful reward contrast can emerge. Moreover, because $\Delta_\theta(z,b)$ depends on the current policy $\pi_\theta$, its relation to heuristic signals shifts during training. Thus, a fixed heuristic proxy will fail not only because it rarely represents recoverability accurately, but also because the target itself evolves during training. This mismatch can be formulated as below.

\begin{comment}
\textbf{Observation 1 (Non-stationary recoverability gap).}
\emph{For a policy $\pi_{\theta_t}$ at step $t$, the recoverability oracle $\mathcal{O}_{\theta_t}(z_t)$ indicates the best attainable intervention gain at $z_t$. For a given intervention strategy $f$, its cumulative recoverability gap $\mathcal{G}_T(f)$ over training can be represented as
\begin{equation}
    \mathcal{G}_T(f)
    =
    \sum_{t=1}^{T}
    \mathbb{E}_{z_t}
    \left[
        \mathcal{O}_{\theta_t}(z_t)
        -
        \mathbb{E}\!\left[\Delta_{\theta_t}(z_t,f(z_t))\mid z_t\right]
    \right].
\end{equation}
A strategy with persistent mismatch from $\mathcal{O}_{\theta_t}$ on a linear number of steps yields $\mathcal{G}_T(f)=\Omega(T)$; existing heuristic intervention methods are prone to this case because their fixed proxy-to-intervention rules cannot correct mismatch from realized outcomes. Effective intervention should instead correct such mismatch and achieve sublinear growth.}
\end{comment}

\begin{definitionbox}
\noindent\textbf{Observation 1 (Non-stationary recoverability gap).}
For a policy $\pi_{\theta_t}$ at step $t$, define the recoverability oracle as $\mathcal{O}_{\theta_t}(z):= \max_{b \in \mathcal{B}} \Delta_{\theta_t}(z,b)$. For a given intervention strategy $f$, its cumulative recoverability gap $G_T(f)$ over training can be represented as: 

\vspace{-8pt}
\begin{equation}
G_T(f) = \sum_{t=1}^{T} \mathbb{E}_{z_t} \left[ \mathcal{O}_{\theta_t}(z_t)
- \Delta_{\theta_t}\!\left(z_t,f(z_t)\right) \right]. 
\end{equation}
\vspace{-8pt}

A strategy with persistent mismatch from $\mathcal{O}_{\theta_t}$ on a
linear number of steps yields $G_T(f)=\Omega(T)$; existing heuristic intervention methods are prone to this case because their fixed proxy-to-intervention rules cannot correct mismatch from realized outcomes. Effective intervention should instead correct such mismatch and achieve sublinear growth.
\end{definitionbox}

To bridge this gap, we therefore ask the following two questions. \textbf{RQ1 (Effectiveness):} \emph{Does \textbf{RAIL} improve final task success over uniform GRPO and heuristic rollout-intervention baselines?} More specifically, \textbf{RQ2 (Adaptivity):} \emph{Can \textbf{RAIL} track changing recoverability better than static heuristic rules? if so, can we show it both empirically and theoretically?}

\textbf{The Non-scalar Gap.}
Another challenge is that effective intervention cannot be fully described by a scalar rollout budget. Once a recoverable state is identified, the intervention must still determine how to reshape the rollout distribution, such as how much budget to spend and how to explore. These choices are coupled, yet prior works either focus mainly on rollout count or treat them largely separate decisions. Thus, the deeper limitation is not only the omission of additional intervention dimensions, but the lack of a principled mechanism for coordinating them. Consequently, such methods can miss recoverability gains that come not from sampling more, but from sampling differently.

\begin{definitionbox}
\textbf{Observation 2 (Non-scalar expressiveness gap).}
Let $\mathcal{B}_{\mathrm{scalar}}\subset\mathcal{B}$ denote a restricted intervention space that varies only one scalar dimension, such as rollout count, while fixing other intervention dimensions;  Let $\mathcal{O}_{\mathcal{B}}(z)$ and $\mathcal{O}_{\mathcal{B}_{\mathrm{scalar}}}(z)$ denote the best attainable recoverability gain under the structured and scalar intervention spaces, which induces an expressiveness gap:
\begin{equation}
    \mathcal{E}_{\mathrm{scalar}}(z)
    =
    \mathcal{O}_{\mathcal{B}}(z)
    -
    \mathcal{O}_{\mathcal{B}_{\mathrm{scalar}}}(z).
\end{equation}
When $\mathcal{E}_{\mathrm{scalar}}(z)>0$, improving $\mathcal{O}_{\mathcal{B}_{\mathrm{scalar}}}(z)$ alone cannot recover the gain from the structured intervention space. Thus, beyond a single scalar intervention dimension, an effective intervention strategy requires an expressive space that captures coupled rollout intervention decisions.
\end{definitionbox}

This gap motivates \textbf{RQ3 (Expressiveness):} \emph{Does modeling intervention as a structured decision space enable \textbf{RAIL} to recover useful training signals, and how do intervention dimensions contribute differently to different tasks?} Finally, another important benefit of rollout intervention is its potential to improve training efficiency by while avoiding saturated or unrecoverable ones. This leads to \textbf{RQ4 (Efficiency):} \emph{Can \textbf{RAIL} use fewer rollouts to achieve efficient policy improvement?}

\vspace{-5pt}
\section{RAIL: Recoverability-Aware Intervention Learning}
\label{sec:method}
\vspace{-5pt}

To address the above challenges, we propose \textbf{Recoverability-Aware Intervention Learning (RAIL)}, a training-time framework that enables policy optimization to receive richer rollout signals while improving efficiency by allocating rollout compute according to learned recoverability. As shown in Figure~\ref{fig:rail_overview}, RAIL couples structured rollout intervention with an online recoverability controller, allowing rollout intervention to adapt with the evolving policy rather than relying on fixed uncertainty or difficulty heuristics. Specifically, the method has two components. Section~\ref{sec:method-learning} (Figure~\ref{fig:rail_overview}.a) formulates recoverability learning as a contextual-bandit problem, trains the controller from intervention traces; Section~\ref{sec:method-deploy} (Figure~\ref{fig:rail_overview}.b) introduces the structured intervention interface and shadow-to-live deployment, enabling trace collection followed by utility-gated intervention during policy optimization.

\begin{figure}[t]
	\centering
    \vspace{-15pt}
	\includegraphics[width=1\linewidth]{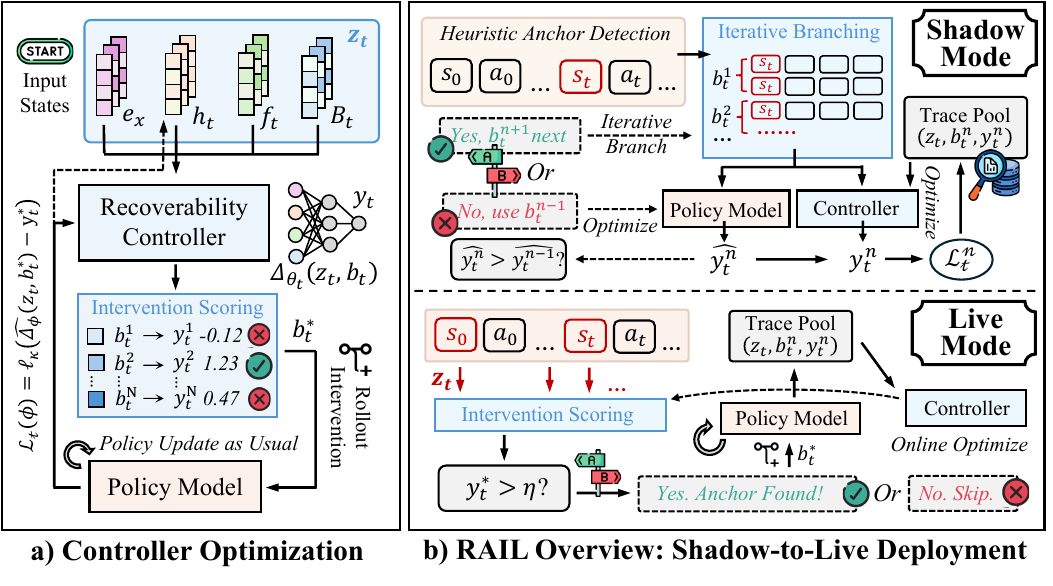}
        \vspace{-20pt}
	\caption{\textbf{Overall Framework of RAIL.}
(a) The recoverability controller predicts intervention gains from state-action traces. (b) RAIL uses a shadow-to-live procedure: shadow mode collects supervision through heuristic iterative branching, while live mode deploys the learned controller for utility-gated rollout intervention and online adaptation as the policy evolves.
    }
    \label{fig:rail_overview}
    \vspace{-15pt}
\end{figure}

\subsection{Learning Recoverability from Intervention Outcomes}
\label{sec:method-learning}

\begin{comment}
\textbf{Recoverability Controller.}
The central challenge in recoverability-aware intervention is that recoverability gain is only observed after acting. Formally, under the current policy $\pi_{\theta_t}$, applying intervention $b_t$ at rollout state $z_t$ yields the post-hoc signal $y_t=\Delta_{\theta_t}(z_t,b_t)$ only after the intervened rollouts are generated. RAIL addresses this observability gap by maintaining a learned \textbf{\emph{recoverability controller}}, trained from intervention traces $(z_t,b_t,y_t)$ rather than fixed proxies such as uncertainty or difficulty. These traces induce a contextual-bandit interaction: at each decision point, the controller observes $z_t$, chooses one intervention $b_t\in\mathcal{B}$, and observes only the utility $y_t$ of the selected intervention. Since the gains of unchosen interventions remain counterfactual, the controller must infer intervention utilities from partial feedback accumulated over training. To make this partial-feedback problem learnable, the controller is parameterized by a score function $\widehat{\Delta}_{\phi_t}(z,b)$, where larger scores indicate larger predicted recoverability gain, and trained by fitting its predicted scores to realized gains:
\end{comment}

\textbf{Recoverability Controller.} The central challenge in recoverability-aware intervention is that applying intervention $b_t$ at state $z_t$ yields only an empirical gain $y_t$, which is only observed after acting. Therefore, RAIL addresses this observability gap by maintaining a learned
\textbf{\emph{recoverability controller}}, trained from intervention traces
$(z_t,b_t,y_t)$.
Formally, under the current policy $\pi_{\theta_t}$, let $\Delta_{\theta_t}(z,b):=\mathbb{E}[y_t \mid z,b,\pi_{\theta_t}]$ denote the expected gain of applying intervention $b$ at state $z$. After applying $b_t$ at $z_t$, RAIL observes finite-sample outcome $y_t=\Delta_{\theta_t}(z_t,b_t)+\xi_t, (\mathbb{E}[\xi_t\mid z_t,b_t,\pi_{\theta_t}]=0),$
once the intervened traces have been generated. These traces induce a contextual-bandit interaction: at each decision point, the controller observes $z_t$, chooses one intervention $b_t\in\mathcal{B}$, and observes only the outcome $y_t$ of the selected intervention. 
% Here, \(z_t\) is the intervention context rather than only the trajectory prefix; it includes the candidate state, environment observation, current rollout-group statistics, uncertainty features, and trajectory position.
Since the gains of unchosen interventions remain counterfactual,
the controller must infer intervention utilities from partial feedback
accumulated over training. To make this partial-feedback problem learnable,
the controller is parameterized by a score function
$\widehat{\Delta}_{\phi_t}(z,b)$, where larger scores indicate larger predicted
recoverability gain, and trained by fitting its predicted scores to the observed gains:

\vspace{-10pt}
\begin{equation}
\label{eq:rail_predictor_obj}
    \mathcal{L}_t(\phi)
    =
    \sum_{j\in\mathcal{I}_t}
    w_{t-j}\,
    \ell_{\kappa}
    \left(
        \widehat{\Delta}_{\phi}(z_j,b_j)-y_j
    \right),
    \quad
    \mathcal{I}_t=\{j:\max(1,t-W+1)\le j\le t\}.
\end{equation}
\vspace{-10pt}

Here, $\mathcal{I}_t$ denotes the recent trace window of size $W$, and $w_{t-j}$ is a non-increasing recency weight that assigns larger weight to more recent traces. The Huber loss $\ell_{\kappa}$ handles noisy empirical gains from finite rollout groups, while the windowed recency weighting helps the controller track the changing recoverability function as $\pi_{\theta_t}$ evolves. We defer the details of the implementation, including contents of state $z_t$, intervention spaces $\mathcal{B}$, exploration schedules and other details to Appendix~\ref{app:rail_details}.

\textbf{Recoverability Regret.}
Having specified how the controller is learned, we next ask if it can avoid the linear growth of the recoverability gap incurred by fixed heuristic strategies. Since $\mathcal{B}$ parameterizes executable interventions, we focus on regret analysis over rounds in which the utility gate triggers intervention. We define the cumulative \emph{recoverability regret} as

\vspace{-15pt}
\begin{equation}
\label{eq:recoverability_regret}
    \mathrm{Reg}_T
    =
    \sum_{t=1}^{T}
    \mathbb{E}_{z_t}
    \left[
        \max_{b\in\mathcal{B}}
        \Delta_{\theta_t}(z_t,b)
        -
        \Delta_{\theta_t}(z_t,b_t)
    \right].
\end{equation}
This is the learned-controller version of the recoverability gap in Observation~1: it compares the controller's selected intervention with the best executable intervention in $\mathcal{B}$ at each intervened round. We therefore aim to show that this regret can grow sublinearly under standard conditions.

% \textbf{Recoverability Regret.}
% Having specified how the controller is learned, we next ask if it can avoid the linear growth of the recoverability gap incurred by fixed heuristic strategies. 
% %Since the controller is learned through contextual-bandit feedback, we define its cumulative \emph{recoverability regret} as 
% {\color{blue}
% Throughout this analysis, $t$ indexes rounds in which an intervention is
% executed, and $b_t \in \mathcal{B}$ denotes the selected intervention. The decision of whether to intervene at all is outside the scope of the regret analysis.
% We define the cumulative recoverability regret as
% }

% \begin{equation}
% \label{eq:recoverability_regret}
%     \mathrm{Reg}_T
%     =
%     \sum_{t=1}^{T}
%     \mathbb{E}_{z_t}
%     \left[
%         \max_{b\in\mathcal{B}}
%         \Delta_{\theta_t}(z_t,b)
%         -
%         \Delta_{\theta_t}(z_t,b_t)
%     \right].
% \end{equation}
% This is the learned-controller version of the recoverability gap in Proposition~1: it compares the controller's selected intervention with the best intervention in $\mathcal{B}$ at %each training step.
% {\color{blue}
% at each executed-intervention round.}
% We therefore aim to show that this regret can grow sublinearly under standard conditions.

\begin{theorem}[Recoverability regret under tracking bound condition]
\label{thm:rail_regret}
Given a finite intervention space $\mathcal{B}$ and bounded recoverability gains, let $D_T$ denote the cumulative policy drift variation of $\Delta_{\theta_t}$, and $A_T$ denote the cumulative approximation error of the controller class. Under finite-action contextual-bandit feedback with sufficient intervention exploration, the regret of the controller satisfies

\vspace{-15pt}
\begin{equation}
    \mathrm{Reg}_T
    \le
    \tilde{O}
    \left(
        \sqrt{T|\mathcal{B}|}+D_T+A_T
    \right).
\end{equation}
\vspace{-15pt}

Consequently, when $D_T+A_T=o(T)$, the controller achieves sublinear recoverability regret.
\end{theorem}
\vspace{-5pt}

The terms $D_T$ and $A_T$ generally align with the same stability and learnability conditions under which policy optimization is expected to make progress. Under these conditions, the controller achieves sublinear regret, so its average recoverability regret decreases asymptotically. A fixed heuristic, however, cannot use outcome feedback to correct proxy mismatch. Thus, even with those assumptions, a persistent per-step mismatch accumulates into linear recoverability regret, which can bias policy optimization toward less informative rollouts, ultimately harming policy optimization. We provide the full assumptions and proof in Appendix~\ref{app:regret}.

\subsection{Structured Intervention and Shadow-to-Live Deployment}
\label{sec:method-deploy}

\textbf{Overview.}
Although the controller scores interventions by $\widehat{\Delta}_{\phi_t}(z,b)$, these scores are not reliable until the controller has observed enough intervention outcomes. RAIL therefore uses a \textbf{shadow-to-live procedure}: a warm-up policy executes structured interventions and records realized gains, allowing the controller to learn from $(z_t,b_t,y_t)$ before it controls live rollout allocation. Once deployed, the controller selects interventions directly and continues updating from new traces as the policy evolves.

\textbf{Shadow Phase and Structured Trace Collection.}
In the shadow phase, RAIL uses a lightweight heuristic warm-up policy to collect intervention traces before live deployment, sampling states where recoverability is plausible and exposing recoverability signals for the controller to learn. Given an initial rollout group, the warm-up policy identifies candidate branch points using high-percentile decision entropy, since such anchors mark states where the policy has not collapsed to a single continuation mode and additional branching is more likely to reveal useful reward contrast. At each selected anchor, the warm-up policy considers a finite executable intervention space
\begin{equation}
\label{eq:rail_intervention_space}
    \mathcal{B}
    =
    \mathcal{M}
    \times
    \mathcal{T},
\end{equation}
where $\mathcal{M}$ is the combination of a discrete set of branch budgets and $\mathcal{T}$ is a discrete set of decoding regimes (e.g. sampling temperatures). This is a minimal instantiation of RAIL's structured intervention interface: the branch budget controls how many additional continuations are generated, while the decoding regimes controls how exploratory those continuations are.

Specifically, at each selected anchor, the warm-up policy applies candidate branches in increasing budget order. It starts with a small branch budget; if the added continuations increase reward variance, the branch is retained and the next larger budget is tried. Otherwise, or once the maximum budget is reached, branching at that anchor stops. This iterative accept-or-stop design \textbf{provides two advantages. First,} it searches for a near-best useful branch budget under the current policy, rather than committing to a fixed amount of extra rollout compute. \textbf{Second,} each accepted or rejected trial becomes a separate intervention outcome, giving the controller multiple within-task training signals rather than a single label from the final rollout group. Alongside budget selection, the warm-up policy also chooses the decoding regime for each trial based on task difficulty and current rollout outcomes. Harder or unresolved tasks are assigned more exploratory decoding regimes, since greater continuation diversity is more likely to uncover alternative successful trajectories. We include the specific branch budgets, decoding regimes and the stopping criteria of shadow phase in Appendix~\ref{app:rail_details}.

For each attempted intervention $b_t=(m_t,\tau_t)$ at state $z_t$, the controller predicts its utility $\widehat{\Delta}_{\phi_t}(z_t,b_t)$ before the branch outcome is observed. After the branch is evaluated, RAIL computes the realized recoverability gain $y_t$ and stores the trace $(z_t,b_t,y_t)$ in the training buffer. The controller is then updated by fitting \eqref{eq:rail_predictor_obj} on this buffer, where older traces receive smaller weights. Thus, the shadow phase provides the initial supervised data to train the controller over the same structured intervention space used in live phase, while keeping early rollout allocation under the warm-up policy.

% \textbf{Shadow-to-Live Transition.}
% RAIL switches from shadow to live deployment using the controller's \emph{prequential} decision quality rather than a fixed hand-set step. For each shadow intervention, the prediction $\widehat{\Delta}_{\phi_t}(z_t,b_t)$ is recorded before observing $y_t$ and before updating the controller, so its accuracy provides an out-of-sample signal of whether recoverability has become learnable. We track rolling sign agreement over the most recent $W$ traces:
% \begin{equation}
% \label{eq:rail_preq_gate}
%     g_t
%     =
%     \frac{1}{W}\sum_{i=t-W+1}^{t}
%     \mathbf{1}\!\left[
%     \operatorname{sign}\widehat{\Delta}_{\phi_i}(z_i,b_i)
%     =
%     \operatorname{sign} y_i
%     \right].
% \end{equation}
% RAIL enters live deployment once $g_t$ stays above the chance-referenced threshold
% $\tfrac{1}{2}+c\sqrt{1/(4W)}$ for $W$ consecutive traces, provided that every intervention in $\mathcal{B}$ has been observed at least once and a minimum warm-up length has elapsed; otherwise, a budget cap forces the transition. We use sign agreement rather than regression error because the live rule depends on choosing interventions with positive utility, while realized gains are noisy and often near zero. Thus, the gate tests whether the controller is reliably better than chance at identifying useful interventions, without requiring convergence to a fixed target. We report $W$, $c$, the warm-up length, and the budget cap in Appendix~\ref{app:rail_details}.

\textbf{Live Deployment with Utility Gating.}
After the shadow phase, the controller takes over intervention selection. At each candidate rollout state $z_t$, it scores all executable interventions and selects the one with the largest predicted recoverability gain: $b_t^{*}
    \in
    \arg\max_{b\in\mathcal{B}}
    \widehat{\Delta}_{\phi_t}(z_t,b)$.
To control how often rollout branching is triggered, RAIL uses a live utility gate $\eta$ and executes the selected intervention only when its predicted utility is sufficiently large:

\vspace{-10pt}
\begin{equation}
\label{eq:rail_live_gate}
    b_t
    =
    \begin{cases}
    b_t^{*}, &
    \text{if }
    \widehat{\Delta}_{\phi_t}(z_t,b_t^{\star}) > \eta,\\
    \varnothing, &
    \text{otherwise}.
    \end{cases}
\end{equation}
\vspace{-10pt}

Larger $\eta$ values make intervention more selective, while smaller values increase branching frequency. When the gate is not passed, RAIL continues the default rollout process without branching. During live training, with probability $\rho_t$, RAIL instead samples an intervention from an exploration distribution over $\mathcal{B}$ with nonzero support (Detailed in Assumption~\ref{assump:exploration}). Since live deployment only forward-scores a finite set of interventions with a lightweight side predictor, it adds negligible overhead relative to language-model rollout generation and environment interaction. Together with the shadow phase, this utility-gated rule turns recoverability learning into a practical online intervention mechanism: RAIL first obtains structured supervision safely, then deploys the controller selectively where predicted recoverability justifies additional rollout compute.
% {\color{blue}
% The deployment gate determines whether to intervene, whereas the regret
% analysis in Section~\ref{sec:method-learning} concerns only the choice among candidate interventions on rounds when the gate is passed.}

% {\color{red}
% \paragraph{Use of Intervened Rollouts in Policy Updates.}
% }

\textbf{Co-evolution with Policy Optimization.}
RAIL integrates with critic-free policy optimization by shaping the rollout groups from which rewards and relative advantages are computed. The controller intervenes at states predicted to yield richer reward contrast, producing more informative trajectories for the policy update. After each update, the changed policy induces new recoverability patterns, and executed interventions provide fresh traces $(z_t,b_t,y_t)$ for further controller training. Thus, RAIL forms an online co-evolution loop: the controller improves rollout allocation for the policy, while the evolving policy supplies new outcomes for the controller to track.

\begin{table*}[t]
\vspace{-15pt}
\centering
\begingroup
\setlength{\tabcolsep}{3.5pt}
\renewcommand{\arraystretch}{1.08}
\resizebox{\textwidth}{!}{
\begin{tabular}{l@{\hspace{-8pt}} c cccc ccc cc}
\toprule
\midrule
\multirow{2}{*}{\textbf{Method}} &
\multicolumn{1}{c}{\textbf{AgentBench OS}} &
\multicolumn{4}{c}{\textbf{AgentBench DB}} &
\multicolumn{3}{c}{\textbf{WebShop}} &
\multicolumn{2}{c}{\textbf{ToolQA Coffee}} \\
\cmidrule(lr){2-2}\cmidrule(lr){3-6}\cmidrule(lr){7-9}\cmidrule(lr){10-11}
& \textbf{Overall SR}
& \textbf{Overall SR} & \textbf{SELECT} & \textbf{INSERT} & \textbf{UPDATE}
& \textbf{Overall SR} & \textbf{L1} & \textbf{L2}
& \textbf{Overall SR} & \textbf{Hard} \\
\midrule

\rowcolor[RGB]{230,230,230}
\multicolumn{11}{c}{\textit{Direct Reasoning}} \\
\midrule
Qwen3-4B
& 23.84{\tiny$\pm$0.50}
& 53.42{\tiny$\pm$0.28} & 42.25{\tiny$\pm$0.50} & 40.82{\tiny$\pm$0.00} & 78.79{\tiny$\pm$0.82}
& 33.12{\tiny$\pm$0.57} & 37.45{\tiny$\pm$0.77} & 28.89{\tiny$\pm$0.69}
& 74.75{\tiny$\pm$0.66} & 55.00{\tiny$\pm$1.22} \\
Qwen3-8B
& 27.68{\tiny$\pm$0.59}
& 56.25{\tiny$\pm$0.95} & 43.25{\tiny$\pm$2.49} & 43.00{\tiny$\pm$0.00} & 82.50{\tiny$\pm$0.50}
& 21.50{\tiny$\pm$0.33} & 29.22{\tiny$\pm$0.77} & 14.20{\tiny$\pm$0.58}
& 75.00{\tiny$\pm$0.61} & 53.00{\tiny$\pm$1.22} \\
Qwen3-14B
& 28.04{\tiny$\pm$0.93}
& 54.17{\tiny$\pm$2.51} & 36.75{\tiny$\pm$3.34} & 45.75{\tiny$\pm$2.28} & 80.00{\tiny$\pm$6.56}
& 36.75{\tiny$\pm$0.43} & \underline{48.66{\tiny$\pm$0.67}} & 25.49{\tiny$\pm$0.80}
& 75.62{\tiny$\pm$0.65} & 51.25{\tiny$\pm$1.30} \\
Qwen3-32B
& 27.68{\tiny$\pm$0.93}
& 59.08{\tiny$\pm$0.14} & 43.50{\tiny$\pm$0.50} & 47.00{\tiny$\pm$0.71} & \textbf{86.75{\tiny$\pm$0.43}}
& 39.70{\tiny$\pm$0.99} & \textbf{49.69{\tiny$\pm$1.14}} & 30.25{\tiny$\pm$2.44}
& 76.62{\tiny$\pm$1.14} & 58.50{\tiny$\pm$1.50} \\

\midrule
\rowcolor[RGB]{230,230,230}
\multicolumn{11}{c}{\textit{GRPO Baselines w/o Rollout Intervention}} \\
\midrule
GRPO-8
& 25.54{\tiny$\pm$0.36}
& 55.42{\tiny$\pm$0.36} & 41.50{\tiny$\pm$1.00} & 45.50{\tiny$\pm$0.58} & 79.25{\tiny$\pm$0.50}
& 36.90{\tiny$\pm$0.57} & 40.64{\tiny$\pm$0.94} & 33.37{\tiny$\pm$0.69}
& 80.00{\tiny$\pm$0.35} & 64.00{\tiny$\pm$0.71} \\
GRPO-16
& 27.32{\tiny$\pm$0.68}
& 56.33{\tiny$\pm$0.41} & 40.75{\tiny$\pm$1.26} & 47.00{\tiny$\pm$0.00} & 81.25{\tiny$\pm$0.50}
& 37.20{\tiny$\pm$0.24} & 44.75{\tiny$\pm$0.67} & 30.06{\tiny$\pm$0.32}
& 81.88{\tiny$\pm$0.22} & 67.75{\tiny$\pm$0.43} \\
GRPO-32
& 27.50{\tiny$\pm$1.70}
& 57.67{\tiny$\pm$0.58} & 43.00{\tiny$\pm$2.00} & 48.50{\tiny$\pm$0.58} & 81.50{\tiny$\pm$0.58}
& 41.65{\tiny$\pm$1.21} & 41.15{\tiny$\pm$1.91} & 42.12{\tiny$\pm$1.08}
& 84.25{\tiny$\pm$1.03} & \underline{75.00{\tiny$\pm$0.00}} \\

\midrule
\rowcolor[RGB]{230,230,230}
\multicolumn{11}{c}{\textit{Step-level Entropy-guided Branching}} \\
\midrule
ARPO-16
& 28.75{\tiny$\pm$0.90}
& 59.83{\tiny$\pm$0.43} & 45.00{\tiny$\pm$0.82} & 49.50{\tiny$\pm$0.58} & \underline{85.00{\tiny$\pm$0.00}}
& 36.40{\tiny$\pm$0.42} & 42.28{\tiny$\pm$0.73} & 30.84{\tiny$\pm$0.17}
& 84.38{\tiny$\pm$0.22} & 74.75{\tiny$\pm$0.43} \\
ARPO-32
& 29.64{\tiny$\pm$0.41}
& \underline{60.25{\tiny$\pm$0.17}} & 45.75{\tiny$\pm$0.50} & \textbf{52.00{\tiny$\pm$0.00}} & 83.00{\tiny$\pm$0.00}
& 40.70{\tiny$\pm$0.41} & 46.09{\tiny$\pm$0.50} & 35.60{\tiny$\pm$0.65}
& 84.25{\tiny$\pm$0.25} & 74.50{\tiny$\pm$0.50} \\
AEPO
& 31.07{\tiny$\pm$0.41}
& 57.08{\tiny$\pm$1.32} & 44.50{\tiny$\pm$0.58} & 47.00{\tiny$\pm$2.16} & 79.75{\tiny$\pm$2.22}
& 41.55{\tiny$\pm$0.95} & 40.53{\tiny$\pm$1.85} & \underline{42.51{\tiny$\pm$0.89}}
& 83.62{\tiny$\pm$0.22} & 74.25{\tiny$\pm$0.43} \\

\midrule
\rowcolor[RGB]{230,230,230}
\multicolumn{11}{c}{\textit{Task-level Allocation \& Optimization}} \\
\midrule
VIP-16
& 28.57{\tiny$\pm$1.93}
& 56.67{\tiny$\pm$0.27} & 42.00{\tiny$\pm$0.82} & 47.00{\tiny$\pm$0.00} & 81.00{\tiny$\pm$0.00}
& 41.60{\tiny$\pm$0.42} & 44.55{\tiny$\pm$1.02} & 38.81{\tiny$\pm$0.17}
& 84.00{\tiny$\pm$0.94} & 73.75{\tiny$\pm$1.64} \\
VIP-32
& 30.71{\tiny$\pm$0.82}
& 58.25{\tiny$\pm$0.32} & \underline{46.50{\tiny$\pm$0.58}} & 48.50{\tiny$\pm$0.58} & 79.75{\tiny$\pm$0.50}
& \underline{44.25{\tiny$\pm$0.22}} & 46.09{\tiny$\pm$0.41} & \underline{42.51{\tiny$\pm$0.32}}
& \underline{84.50{\tiny$\pm$0.00}} & 74.00{\tiny$\pm$0.00} \\
Tree-GRPO
& 28.57{\tiny$\pm$0.58}
& 56.67{\tiny$\pm$0.38} & 42.75{\tiny$\pm$0.96} & 45.75{\tiny$\pm$0.50} & 81.50{\tiny$\pm$0.58}
& 38.30{\tiny$\pm$0.41} & 45.47{\tiny$\pm$0.68} & 31.52{\tiny$\pm$0.48}
& 83.25{\tiny$\pm$0.75} & 74.50{\tiny$\pm$0.50} \\
TAMPO
& \underline{31.25{\tiny$\pm$0.68}}
& 58.08{\tiny$\pm$0.42} & 43.00{\tiny$\pm$0.82} & 47.25{\tiny$\pm$0.96} & 84.00{\tiny$\pm$0.00}
& 40.95{\tiny$\pm$0.83} & 45.06{\tiny$\pm$0.74} & 37.06{\tiny$\pm$1.21}
& 81.25{\tiny$\pm$0.56} & 66.50{\tiny$\pm$1.12} \\

\midrule
\midrule
\rowcolor[RGB]{222,230,241}
\textbf{RAIL}
& \textbf{33.30{\tiny$\pm$0.35}}
& \textbf{61.67{\tiny$\pm$0.41}} & \textbf{50.50{\tiny$\pm$1.29}} & \underline{51.00{\tiny$\pm$0.82}} & 83.50{\tiny$\pm$0.58}
& \textbf{45.85{\tiny$\pm$0.38}} & 46.19{\tiny$\pm$0.53} & \textbf{45.53{\tiny$\pm$0.95}}
& \textbf{86.75{\tiny$\pm$0.56}} & \textbf{78.50{\tiny$\pm$1.12}} \\
\bottomrule
\end{tabular}
}
\endgroup
\vspace{-5pt}
\caption{Main results on agentic reasoning benchmarks. We report success rate (SR) across four benchmarks and fine-grained splits. Baselines are trained with Qwen3-4B-Instruct. Best results are bolded and second-best results are underlined. Numerical suffixes denote the max rollout budget.}
\label{tab:main_results}
\vspace{-20pt}
\end{table*}

\vspace{-10pt}
\section{Experiments}
\label{sec:exp}
\vspace{-5pt}
\subsection{Experiment Setup}
\label{sec:exp_setup}
\vspace{-5pt}

\textbf{Benchmarks.}
We evaluate RAIL on four benchmarks in agentic reasoning spanning classic tool-interaction types. AgentBench-OS and AgentBench-DB \citep{liu2024agentbench} evaluate reasoning in interactive operating-system and database environments. WebShop~\citep{yao2022webshop} evaluates web-based product search and purchase, requiring agents to navigate webpages. ToolQA-Coffee~\citep{zhuang2023toolqa} evaluates tabular data, where the agent uses a Python interpreter to perform data analysis. Across all benchmarks, we report success rate (SR) for the same four consecutive seeds with its average. Detailed dataset descriptions and configurations are provided in Appendix~\ref{app:benchmark}.

\textbf{Baselines. }
To align with our scope, we focus on the following baselines related to rollout intervention, including: (i) \textbf{GRPO}, without adaptive intervention \citep{shao2024deepseekmath}; (ii) \textbf{step-level branching}, such as ARPO \citep{dong2025agentic} and AEPO \citep{dong2026agentic}, which spawn extra rollouts from high-entropy states; (iii) \textbf{task-level budget allocation}, e.g., VIP \citep{nguyen2026adaptive} and TAMPO \citep{dang2026temperature} which allocates rollouts or changes decoding temperatures from a predicted per-task success signals; and Tree-GRPO \citep{ji2025tree}, which share trajectory prefixes to derive step-level advantages. All methods are evaluated under a unified protocol with Qwen3-4B-Instruct. More implementation and adaptation details are provided in Appendix~\ref{app:baselines}.

\begin{figure}[t]
	\centering
    \vspace{-20pt}
	\includegraphics[width=1\linewidth]{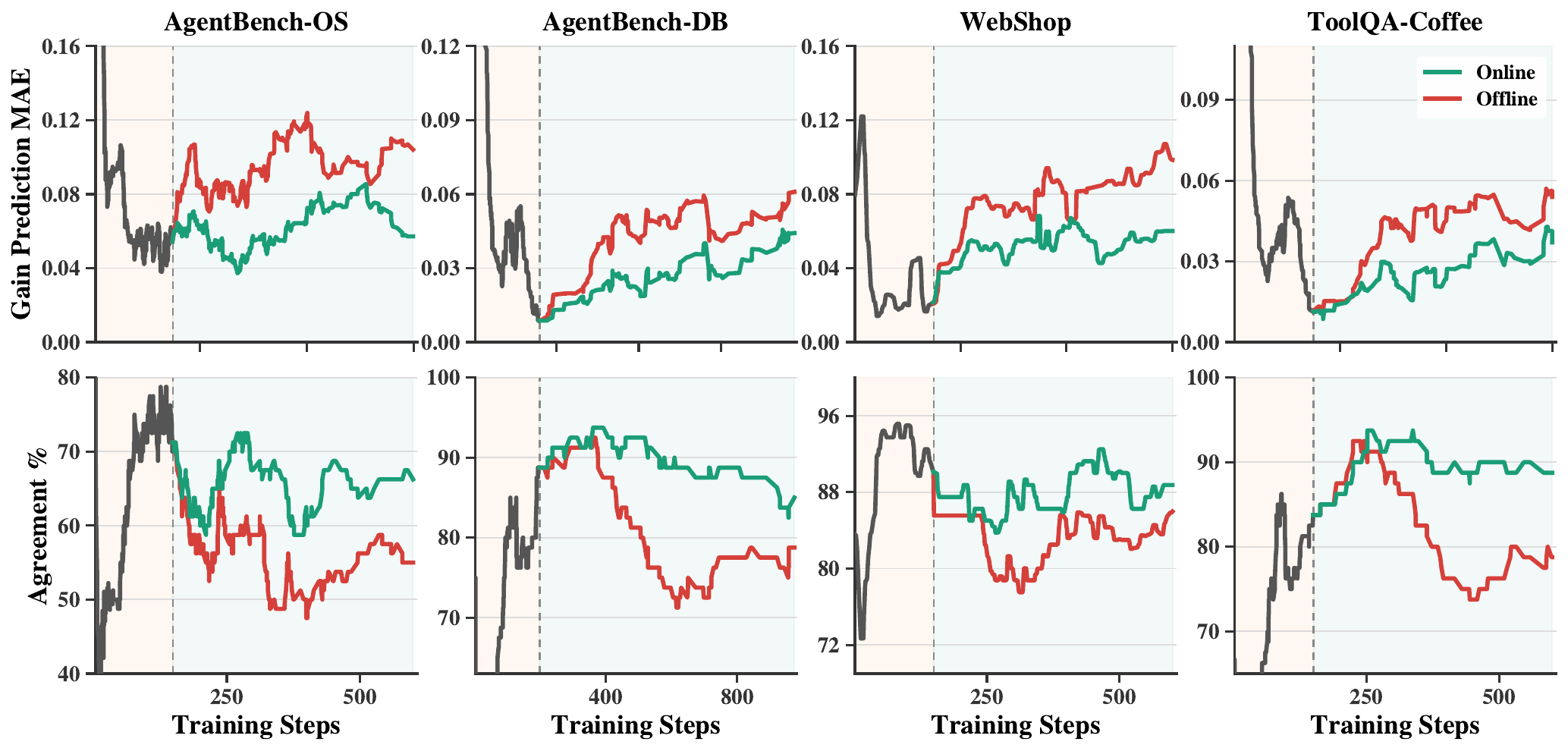}
        \vspace{-20pt}
	\caption{\textbf{Adaptivity of the recoverability controller.} In live mode after the shadow phase (black), online updating (green) yields lower MAE and higher sign agreement than a frozen controller (red), indicating that recoverability evolves with the policy and benefits from online tracking.
    }
    \vspace{-15pt}    
    \label{fig:adapt}
\end{figure}

\vspace{-5pt}
\subsection{Experiment Results}

\textbf{RQ1: Effectiveness.}
We present the main results of RAIL in Table~\ref{tab:main_results}. Overall, 
RAIL consistently outperforms all baselines across four agentic reasoning benchmarks, 
% {\color{red} and the best result on most fine-grained splits}, 
demonstrating that learning rollout intervention from realized recoverability gains provides more effective optimization signals than uniform sampling or fixed heuristic allocation. Moreover, several insights emerge. \textbf{First, adaptive rollout intervention is essential for agentic GRPO.} 
Uniform GRPO improves over direct reasoning, but the gains saturate as the rollout budget increases. This indicates that simply increasing rollout count is insufficient. \textbf{Second, heuristic intervention improves GRPO but is not consistently reliable.} While methods such as ARPO and VIP perform strong on some tasks, no heuristic baseline dominates across environments. This supports our motivation that these heuristic rules are only indirect proxies for recoverability, and their utility varies across task types. \textbf{Third, RAIL gains come from learning where and how to intervene.} The gains are particularly clear on difficult splits (e.g. ToolQA-Hard, WebShop-L2), where RAIL improves over the best baseline by a large margin. \textbf{Overall,} these results show that recoverability-aware intervention produces more informative rollout groups, leading to stronger and more consistent policy improvement in agentic RL.

\textbf{RQ2: Adaptability.}
To answer RQ2, we compare RAIL's online controller with an offline variant whose controller is trained during the shadow phase but frozen after entering live mode. Figure~\ref{fig:adapt} reports prediction MAE and sign agreement between predicted and realized intervention gains across benchmarks. Averaged over the live phase, online updating reduces MAE and improves sign agreement consistently across benchmarks. Lower MAE indicates more accurate recoverability estimation, while higher sign agreement indicates better judgment of whether an intervention is beneficial. We notice \textbf{a fixed controller track policy-induced non-stationarity poorly.} 
\begin{wrapfigure}{t}{0.5\textwidth}
\vspace{-8pt}
\centering
\includegraphics[width=\linewidth]{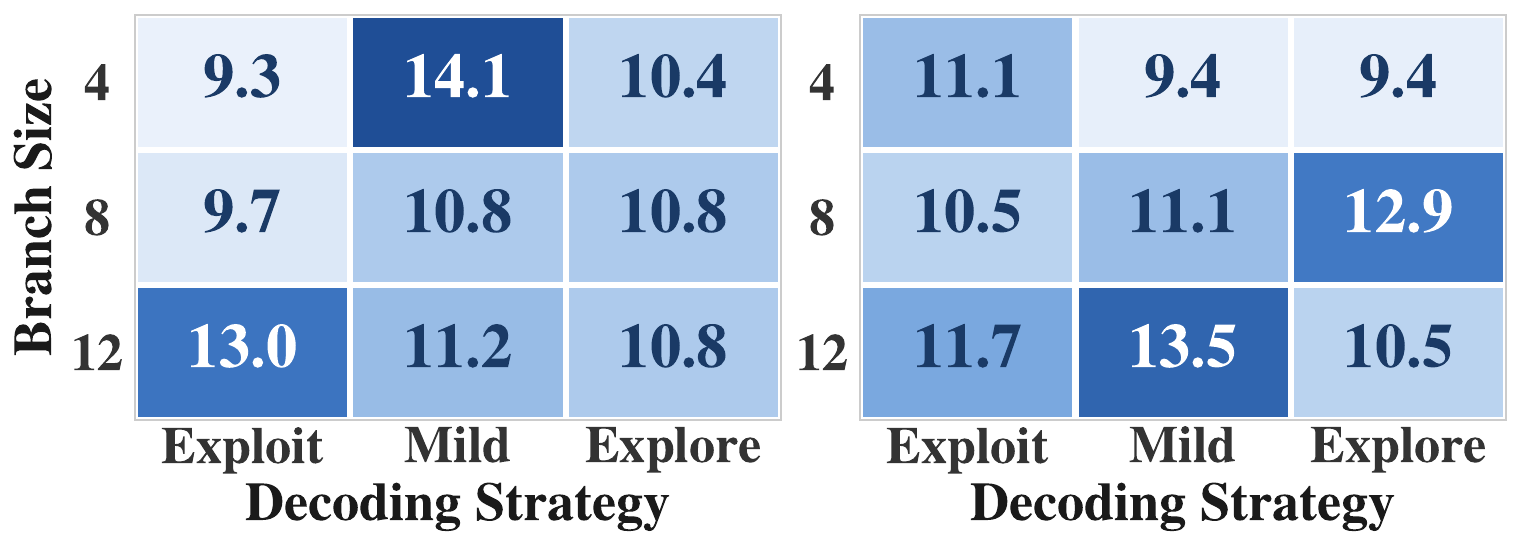}
\vspace{-20pt}
\caption{Learned branching preferences across AgentBench-OS (Left) and AgentBench-DB (Right). Each cell reports intervention frequency (\%), normalized to 100\% per panel. RAIL demonstrate distinct intervention patterns instead of collapsing to a universal choice.}
\label{fig:heatmap}
\vspace{-15pt}
\end{wrapfigure}
After entering live mode, the offline controller exhibits larger prediction errors and declining sign agreement, which showcases that recoverability is not a stationary property of a task or state: as the policy changes, the same uncertainty or branching pattern can correspond to different intervention gains, and continuing to update the controller from newly observed intervention traces benefits the adaptation to the evolving policy. \textbf{Overall,} these results support RAIL's core goal: rollout intervention should not be fixed, as effective intervention depends on recoverability signals that shift during training.

\textbf{RQ3: Expressiveness.}
To answer RQ3, we examine how RAIL benefits from modeling rollout intervention as a structured decision. Figure~\ref{fig:heatmap} visualizes the empirical
intervention distribution selected by the learned controller. RAIL assigns non-trivial probability to multiple branch-size and decoding-regime pairs, rather than collapsing to a single intervention. Meanwhile, the preferred regions differ across environments, which supports our hypothesis that useful rollout
intervention is task-dependent and should be learned from realized traces. To further isolate this effect, Table~\ref{tab:ablation} reports controlled ablations where all training settings are kept unchanged except the intervention space or the shadow/live split and we have several insights: First, scalar intervention is insufficient. The scalar variants underperform the full structured space,
% {\color{red} Although scalar variants can be competitive in some settings (Opt.=4), the full RAIL space achieves stronger overall performance,}
showing that a non-scalar action space to coordinate benefits our task. Second, intervention-space granularity exhibits a trade-off. Increasing scalar options can help the process, but not uniformly true across all cases. Finally, although longer shadow training provides more initial traces, it also delays live deployment and shortens the period in which intervention learning co-evolves with the policy. \textbf{We introduce a variant here with dynamic shadow gate, which we elaborate the design in Appendix~\ref{app:rail_details}.} Overall, these results validate RAIL's core design: effective rollout intervention requires a compact but structured action space, learned online from realized intervention outcomes.

\begin{figure}[t]
	\centering
    \vspace{-20pt}
	\includegraphics[width=1\linewidth]{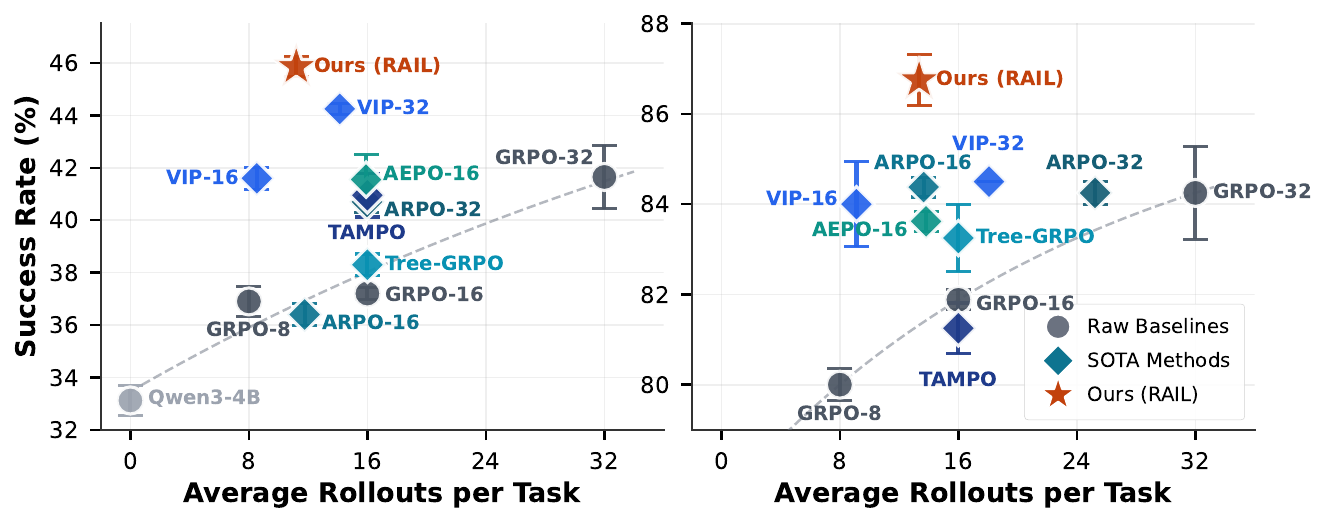}
        \vspace{-20pt}
	\caption{Performance-rollout Results Preview on WebShop (Left) and ToolQA Coffee (Right). RAIL achieves the highest success rates with comparable or fewer rollouts than other intervention baselines.
    }
    \vspace{-5pt}
        
    \label{fig:efficient}
\end{figure}

\begin{table}[t]
\centering
\vspace{-6pt}

\begin{minipage}[t]{0.50\textwidth}
\centering
\setlength{\tabcolsep}{2.1pt}
\renewcommand{\arraystretch}{1.06}
\footnotesize
\resizebox{\linewidth}{!}{
\begin{tabular}{c cc cc}
\toprule
\textbf{Method}
& \textbf{OS}
& \textbf{$\Delta$}
& \textbf{DB}
& \textbf{$\Delta$} \\
\midrule\midrule

\rowcolor[RGB]{230,230,230}
\multicolumn{5}{c}{\textit{Component Removal}} \\
\midrule
w/o Branch.
& 30.89{\tiny$\pm$0.31}
& \textcolor{red!70!black}{\small$\downarrow$7.24\%}
& 59.67{\tiny$\pm$0.27}
& \textcolor{red!70!black}{\small$\downarrow$3.24\%} \\

w/o Decod.
& 31.79{\tiny$\pm$1.07}
& \textcolor{red!70!black}{\small$\downarrow$4.53\%}
& 59.92{\tiny$\pm$0.17}
& \textcolor{red!70!black}{\small$\downarrow$2.84\%} \\

\midrule
\rowcolor[RGB]{230,230,230}
\multicolumn{5}{c}{\textit{Scalar Intervention}} \\
\midrule
Opt.=2
& 29.29{\tiny$\pm$1.01}
& \textcolor{red!70!black}{\small$\downarrow$12.04\%}
& 60.83{\tiny$\pm$0.43}
& \textcolor{red!70!black}{\small$\downarrow$1.36\%} \\

Opt.=4
& \underline{34.11{\tiny$\pm$0.78}}
& \textcolor{green!50!black}{\small$\uparrow$2.43\%}
& 59.75{\tiny$\pm$0.50}
& \textcolor{red!70!black}{\small$\downarrow$3.11\%} \\

\midrule
\rowcolor[RGB]{230,230,230}
\multicolumn{5}{c}{\textit{Shadow Length}} \\
\midrule
Shadow 1/3
& 30.36{\tiny$\pm$1.07}
& \textcolor{red!70!black}{\small$\downarrow$8.83\%}
& 59.58{\tiny$\pm$0.17}
& \textcolor{red!70!black}{\small$\downarrow$3.39\%} \\

Shadow 1/2
& 31.07{\tiny$\pm$0.80}
& \textcolor{red!70!black}{\small$\downarrow$6.70\%}
& 59.25{\tiny$\pm$0.42}
& \textcolor{red!70!black}{\small$\downarrow$3.92\%} \\

Dynamic
& 32.68{\tiny$\pm$0.93}
& \textcolor{red!70!black}{\small$\downarrow$1.86\%}
& 59.83{\tiny$\pm$0.33}
& \textcolor{red!70!black}{\small$\downarrow$2.98\%} \\

\bottomrule
\end{tabular}
}
\vspace{-4pt}
\caption{Ablation study on AgentBench.}
\label{tab:ablation}
\end{minipage}
\hfill
\begin{minipage}[t]{0.47\textwidth}
\centering
\setlength{\tabcolsep}{3pt}
\renewcommand{\arraystretch}{1.06}
\footnotesize
\resizebox{\linewidth}{!}{
\begin{tabular}{c cccc}
\toprule
\textbf{Method} 
& \textbf{OS} 
& \textbf{DB} 
& \textbf{WebShop} 
& \textbf{ToolQA} \\
\midrule\midrule

\rowcolor[RGB]{230,230,230}
\multicolumn{5}{c}{\textit{Step-level Branching}} \\
\midrule
ARPO-16 & 14.68 & 12.95 & 11.76 & 13.66 \\
ARPO-32 & 30.04 & 24.24 & 25.97 & 25.23 \\
AEPO    & 12.43 & 13.85 & 15.91 & 13.82 \\

\midrule
\rowcolor[RGB]{230,230,230}
\multicolumn{5}{c}{\textit{Task-level Allocation}} \\
\midrule
VIP-16
& \underline{11.09} & \textbf{9.18} & \textbf{8.53} & \textbf{9.12} \\
VIP-32
& 24.11 & 14.88 & 14.13 & 18.07 \\
Tree-GRPO
& 16 & 16 & 16 & 16 \\
TAMPO
& 16 & 16 & 16 & 16 \\

\midrule\midrule
\rowcolor[RGB]{222,230,241}
\textbf{RAIL}
& \textbf{10.86} & \underline{12.03} & \underline{11.20} & \underline{13.35} \\
\bottomrule
\end{tabular}
}
\vspace{-4pt}
\caption{Average rollout across benchmarks.}
\label{tab:efficiency}
\end{minipage}

\vspace{-20pt}
\end{table}

\textbf{RQ4: Efficiency.}
Finally, Table~\ref{tab:efficiency} reports the average rollout budget used by each method. RAIL uses substantially fewer rollouts than high-budget adaptive baselines. Together with Table~\ref{tab:main_results}, Figures~\ref{fig:opening} and~\ref{fig:efficient} show a clear performance-rollout advantage: 
%RAIL achieves higher success rates while requiring fewer rollouts. 
% {\color{red} RAIL achieves higher success rates than low-budget baselines while avoiding the large rollout costs of high-budget baselines (e.g. ARPO, VIP).}
%In particular, 
RAIL outperforms GRPO-32 across all overall metrics while using less than half of its rollout budget. Compared with adaptive baselines, RAIL also avoids the large budgets used by ARPO-32 and VIP-32 while achieving stronger performance. \textbf{Overall,} these results indicate that recoverability-aware intervention improves efficiency by directing rollout compute toward states where additional exploration is likely to produce useful training signals.

\section{Conclusion}

We introduced RAIL, a recoverability-aware framework that turns rollout generation from a fixed sampling procedure into a learnable intervention process. By learning from realized intervention gains, RAIL adapts to policy-induced non-stationarity and selects structured interventions over where and how to branch. Across agentic reasoning benchmarks, RAIL consistently improves effectiveness, adaptivity, expressiveness and efficiency over uniform GRPO and heuristic intervention baselines. These results suggest that the rollout-generation process itself should be treated as an optimization object, enabling post-training to learn from stronger and less redundant signals.

\newpage
\bibliography{iclr2026_conference}

@article{nguyen2026adaptive,
  title={Adaptive Rollout Allocation for Online Reinforcement Learning with Verifiable Rewards},
  author={Nguyen, Hieu Trung and Nguyen, Bao and Ma, Wenao and Zhao, Yuzhi and She, Ruifeng and Nguyen, Viet Anh},
  journal={arXiv preprint arXiv:2602.01601},
  year={2026}
}

@article{fang2026allocate,
  title={How to Allocate, How to Learn? Dynamic Rollout Allocation and Advantage Modulation for Policy Optimization},
  author={Fang, Yangyi and Lin, Jiaye and Fu, Xiaoliang and Qin, Cong and Shi, Haolin and Hu, Chaowen and Pan, Lu and Zeng, Ke and Cai, Xunliang},
  journal={arXiv preprint arXiv:2602.19208},
  year={2026}
}

@article{yao2026coba,
  title={CoBA-RL: Capability-Oriented Budget Allocation for Reinforcement Learning in LLMs},
  author={Yao, Zhiyuan and Zhang, Yi-Kai and Chen, Yuxin and Sun, Yueqing and Xu, Zishan and Yang, Yu and Hu, Tianhao and Gu, Qi and Su, Hui and Cai, Xunliang},
  journal={arXiv preprint arXiv:2602.03048},
  year={2026}
}

@article{panaganti2026group,
  title={Group Distributionally Robust Optimization-Driven Reinforcement Learning for LLM Reasoning},
  author={Panaganti, Kishan and Liang, Zhenwen and Yu, Wenhao and Mi, Haitao and Yu, Dong},
  journal={arXiv preprint arXiv:2601.19280},
  year={2026}
}

@article{yang2025depth,
  title={Depth-breadth synergy in rlvr: Unlocking llm reasoning gains with adaptive exploration},
  author={Yang, Zhicheng and Guo, Zhijiang and Huang, Yinya and Wang, Yongxin and Xie, Dongchun and Li, Hanhui and Wang, Yiwei and Liang, Xiaodan and Tang, Jing},
  journal={arXiv preprint arXiv:2508.13755},
  year={2025}
}

@article{xiong2025reinforce,
  title={Reinforce-Ada: An Adaptive Sampling Framework under Non-linear RL Objectives},
  author={Xiong, Wei and Ye, Chenlu and Liao, Baohao and Dong, Hanze and Xu, Xinxing and Monz, Christof and Bian, Jiang and Jiang, Nan and Zhang, Tong},
  journal={arXiv preprint arXiv:2510.04996},
  year={2025}
}

@article{zhao2026training,
  title={Training Multi-Turn Search Agent via Contrastive Dynamic Branch Sampling},
  author={Zhao, Yubao and Huang, Weiquan and Wang, Sudong and Zhao, Ruochen and Chen, Chen and Shu, Yao and Qin, Chengwei},
  journal={arXiv preprint arXiv:2602.03719},
  year={2026}
}

@article{wang2026ragen, 
  title={RAGEN-2: Reasoning Collapse in Agentic RL},
  author={Wang, Zihan and Gui, Chi and Jin, Xing and Wang, Qineng and Liu, Licheng and Wang, Kangrui and Chen, Shiqi and Li, Linjie and Yang, Zhengyuan and Zhang, Pingyue and others},
  journal={arXiv preprint arXiv:2604.06268},
  year={2026}
}

@article{duan2025uprop,
  title={Uprop: Investigating the uncertainty propagation of llms in multi-step agentic decision-making},
  author={Duan, Jinhao and Diffenderfer, James and Madireddy, Sandeep and Chen, Tianlong and Kailkhura, Bhavya and Xu, Kaidi},
  journal={arXiv preprint arXiv:2506.17419},
  year={2025}
}

@article{zheng2025first,
  title={First return, entropy-eliciting explore},
  author={Zheng, Tianyu and Xing, Tianshun and Gu, Qingshui and Liang, Taoran and Qu, Xingwei and Zhou, Xin and Li, Yizhi and Wen, Zhoufutu and Lin, Chenghua and Huang, Wenhao and others},
  journal={arXiv preprint arXiv:2507.07017},
  year={2025}
}

@article{wei2026entropy,
  title={Entropy-Tree: Tree-Based Decoding with Entropy-Guided Exploration},
  author={Wei, Longxuan and Zhang, Yubo and Zhang, Zijiao and Wang, Zhihu and Zhao, Shiwan and Huang, Tianyu and Zhao, Huiting and Liu, Chenfei and Zhang, Shenao and Yan, Junchi},
  journal={arXiv preprint arXiv:2601.15296},
  year={2026}
}

@article{dong2025agentic,
  title={Agentic reinforced policy optimization},
  author={Dong, Guanting and Mao, Hangyu and Ma, Kai and Bao, Licheng and Chen, Yifei and Wang, Zhongyuan and Chen, Zhongxia and Du, Jiazhen and Wang, Huiyang and Zhang, Fuzheng and others},
  journal={arXiv preprint arXiv:2507.19849},
  year={2025}
}

@article{dong2026agentic,
  title={Agentic entropy-balanced policy optimization},
  author={Dong, Guanting and Bao, Licheng and Wang, Zhongyuan and Zhao, Kangzhi and Li, Xiaoxi and Jin, Jiajie and Yang, Jinghan and Mao, Hangyu and Zhang, Fuzheng and Gai, Kun and others},
  journal={arXiv preprint arXiv:2510.14545},
  year={2025}
}

@article{ji2025tree,
  title={Tree search for llm agent reinforcement learning},
  author={Ji, Yuxiang and Ma, Ziyu and Wang, Yong and Chen, Guanhua and Chu, Xiangxiang and Wu, Liaoni},
  journal={arXiv preprint arXiv:2509.21240},
  year={2025}
}

@article{shao2024deepseekmath,
  title={Deepseekmath: Pushing the limits of mathematical reasoning in open language models},
  author={Shao, Zhihong and Wang, Peiyi and Zhu, Qihao and Xu, Runxin and Song, Junxiao and Bi, Xiao and Zhang, Haowei and Zhang, Mingchuan and Li, YK and Wu, Yang and others},
  journal={arXiv preprint arXiv:2402.03300},
  year={2024}
}

@book{sutton1998reinforcement,
  title={Reinforcement learning: An introduction},
  author={Sutton, Richard S and Barto, Andrew G and others},
  year={1998},
  publisher={MIT press Cambridge}
}

@article{dang2026temperature,
  title={Temperature as a Meta-Policy: Adaptive Temperature in LLM Reinforcement Learning},
  author={Dang, Haoran and Lan, Cuiling and Wan, Hai and Zhao, Xibin and Lu, Yan},
  journal={arXiv preprint arXiv:2602.11779},
  year={2026}
}

@article{tong2024dart,
  title={Dart-math: Difficulty-aware rejection tuning for mathematical problem-solving},
  author={Tong, Yuxuan and Zhang, Xiwen and Wang, Rui and Wu, Ruidong and He, Junxian},
  journal={Advances in Neural Information Processing Systems},
  volume={37},
  pages={7821--7846},
  year={2024}
}

@article{rafailov2023direct,
  title={Direct preference optimization: Your language model is secretly a reward model},
  author={Rafailov, Rafael and Sharma, Archit and Mitchell, Eric and Manning, Christopher D and Ermon, Stefano and Finn, Chelsea},
  journal={Advances in neural information processing systems},
  volume={36},
  pages={53728--53741},
  year={2023}
}

@article{liu2025understanding,
  title={Understanding r1-zero-like training: A critical perspective},
  author={Liu, Zichen and Chen, Changyu and Li, Wenjun and Qi, Penghui and Pang, Tianyu and Du, Chao and Lee, Wee Sun and Lin, Min},
  journal={arXiv preprint arXiv:2503.20783},
  year={2025}
}

@inproceedings{ahmadian2024back,
  title={Back to basics: Revisiting REINFORCE-style optimization for learning from human feedback in LLMs},
  author={Ahmadian, Arash and Cremer, Chris and Gall{\'e}, Matthias and Fadaee, Marzieh and Kreutzer, Julia and Pietquin, Olivier and {\"U}st{\"u}n, Ahmet and Hooker, Sara},
  booktitle={Proceedings of the 62nd Annual Meeting of the Association for Computational Linguistics (Volume 1: Long Papers)},
  pages={12248--12267},
  year={2024}
}

@article{auer2002finite,
  title={Finite-time analysis of the multiarmed bandit problem},
  author={Auer, Peter and Cesa-Bianchi, Nicolo and Fischer, Paul},
  journal={Machine learning},
  volume={47},
  number={2},
  pages={235--256},
  year={2002},
  publisher={Springer}
}

@article{auer2002nonstochastic,
  title={The nonstochastic multiarmed bandit problem},
  author={Auer, Peter and Cesa-Bianchi, Nicolo and Freund, Yoav and Schapire, Robert E},
  journal={SIAM journal on computing},
  volume={32},
  number={1},
  pages={48--77},
  year={2002},
  publisher={SIAM}
}

@article{langford2008epoch,
  title={The epoch-greedy algorithm for multi-armed bandits with side information},
  author={Langford, John and Zhang, Tong},
  journal={Advances in neural information processing systems},
  volume={20},
  year={2007}
}

@inproceedings{agarwal2014taming,
  title={Taming the monster: A fast and simple algorithm for contextual bandits},
  author={Agarwal, Alekh and Hsu, Daniel and Kale, Satyen and Langford, John and Li, Lihong and Schapire, Robert},
  booktitle={International conference on machine learning},
  pages={1638--1646},
  year={2014},
  organization={PMLR}
}

@inproceedings{garivier2011upper,
  title={On upper-confidence bound policies for switching bandit problems},
  author={Garivier, Aur{\'e}lien and Moulines, Eric},
  booktitle={International conference on algorithmic learning theory},
  pages={174--188},
  year={2011},
  organization={Springer}
}

@article{besbes2014stochastic,
  title={Stochastic multi-armed-bandit problem with non-stationary rewards},
  author={Besbes, Omar and Gur, Yonatan and Zeevi, Assaf},
  journal={Advances in neural information processing systems},
  volume={27},
  year={2014}
}

@article{besbes2015nonstationary,
  title={Non-stationary stochastic optimization},
  author={Besbes, Omar and Gur, Yonatan and Zeevi, Assaf},
  journal={Operations research},
  volume={63},
  number={5},
  pages={1227--1244},
  year={2015},
  publisher={INFORMS}
}

@inproceedings{liu2024agentbench,
  title={Agentbench: Evaluating llms as agents},
  author={Liu, Xiao and Yu, Hao and Zhang, Hanchen and Xu, Yifan and Lei, Xuanyu and Lai, Hanyu and Gu, Yu and Ding, Hangliang and Men, Kaiwen and Yang, Kejuan and others},
  booktitle={International Conference on Learning Representations},
  volume={2024},
  pages={52989--53046},
  year={2024}
}

@article{yao2022webshop,
  title={Webshop: Towards scalable real-world web interaction with grounded language agents},
  author={Yao, Shunyu and Chen, Howard and Yang, John and Narasimhan, Karthik},
  journal={Advances in Neural Information Processing Systems},
  volume={35},
  pages={20744--20757},
  year={2022}
}

@article{zhuang2023toolqa,
  title={Toolqa: A dataset for llm question answering with external tools},
  author={Zhuang, Yuchen and Yu, Yue and Wang, Kuan and Sun, Haotian and Zhang, Chao},
  journal={Advances in Neural Information Processing Systems},
  volume={36},
  pages={50117--50143},
  year={2023}
}

@inproceedings{zhang2026mapro,
  title={MAPRO: Recasting Multi-Agent Prompt Optimization as Maximum a Posteriori Inference},
  author={Zhang, Zheyuan and Ge, Lin and Li, Hongjiang and Zhu, Weicheng and Zhang, Chuxu and Ye, Yanfang},
  booktitle={Findings of the Association for Computational Linguistics: EACL 2026},
  pages={4458--4480},
  year={2026}
}

@article{zhang2025agentrouter,
  title={AgentRouter: A Knowledge-Graph-Guided LLM Router for Collaborative Multi-Agent Question Answering},
  author={Zhang, Zheyuan and Shi, Kaiwen and Yuan, Zhengqing and Wang, Zehong and Ma, Tianyi and Murugesan, Keerthiram and Galassi, Vincent and Zhang, Chuxu and Ye, Yanfang},
  journal={arXiv preprint arXiv:2510.05445},
  year={2025}
}

@inproceedings{shi2026ng,
  title={NG-Router: Graph-Supervised Multi-Agent Collaboration for Nutrition Question Answering},
  author={Shi, Kaiwen and Zhang, Zheyuan and Yuan, Zhengqing and Murugesan, Keerthiram and Galassi, Vincent and Zhang, Chuxu and Ye, Yanfang},
  booktitle={Proceedings of the 19th Conference of the European Chapter of the Association for Computational Linguistics (Volume 1: Long Papers)},
  pages={7508--7527},
  year={2026}
}

@article{huang2026evolverouter,
  title={EvolveRouter: Co-Evolving Routing and Prompt for Multi-Agent Question Answering},
  author={Huang, Jiatan and Zhang, Zheyuan and Shi, Kaiwen and Ye, Yanfang and Zhang, Chuxu},
  journal={arXiv preprint arXiv:2604.05149},
  year={2026}
}

@article{bao2026drift,
  title={Drift-Bench: Diagnosing Cooperative Breakdowns in LLM Agents under Input Faults via Multi-Turn Interaction},
  author={Bao, Han and Zhang, Zheyuan and Jing, Pengcheng and Yuan, Zhengqing and Shi, Kaiwen and Ye, Yanfang},
  journal={arXiv preprint arXiv:2602.02455},
  year={2026}
}

@article{zhang2026semantic,
  title={Why semantic entropy fails: Geometry-aware and calibrated uncertainty for policy optimization},
  author={Zhang, Zheyuan and Shi, Kaiwen and Bao, Han and Wang, Zehong and Ma, Tianyi and Ye, Yanfang},
  journal={arXiv preprint arXiv:2605.21801},
  year={2026}
}

@article{shi2026sage,
  title={SAGE: Answer-Conditioned Uncertainty Targets for Verbal Uncertainty Alignment},
  author={Shi, Kaiwen and Zhang, Zheyuan and Ye, Yanfang},
  journal={arXiv preprint arXiv:2606.11512},
  year={2026}
}

@article{ye2025llms4all,
  title={Llms4all: A review of large language models across academic disciplines},
  author={Ye, Yanfang and Zhang, Zheyuan and Ma, Tianyi and Wang, Zehong and Li, Yiyang and Hou, Shifu and Sun, Weixiang and Shi, Kaiwen and Ma, Yijun and Song, Wei and others},
  journal={arXiv preprint arXiv:2509.19580},
  year={2025}
}

@article{ma2024agentboard,
  title={Agentboard: An analytical evaluation board of multi-turn llm agents},
  author={Ma, Chang and Zhang, Junlei and Zhu, Zhihao and Yang, Cheng and Yang, Yujiu and Jin, Yaohui and Lan, Zhenzhong and Kong, Lingpeng and He, Junxian},
  journal={Advances in neural information processing systems},
  volume={37},
  pages={74325--74362},
  year={2024}
}

@article{zou2026trace,
  title   = {{TRACE}: A Unified Rollout Budget Allocation Framework for Efficient Agentic Reinforcement Learning},
  author  = {Zou, Heming and Wang, Qi and Qu, Yun and Jiang, Yuhang and Cai, Lizhou and Mao, Yixiu and Peng, Ru and Xu, Xin and Liu, Weijie and Yang, Kai and Yang, Saiyong and Ji, Xiangyang},
  journal = {arXiv preprint arXiv:2606.11119},
  year    = {2026}
}

@article{han2026three,
  title   = {{3SPO}: State-Score-Supervised Policy Optimization for LLM Agents},
  author  = {Han, Yu and Li, Kailing and Jiao, Yang and Dai, Yulin and Fu, Yuqian and Zhuo, Linhai and Qian, Tianwen},
  journal = {arXiv preprint arXiv:2606.09961},
  year    = {2026}
}

@article{zhang2026igrpo,
  title   = {Information Gain-Based Rollout Policy Optimization: An Adaptive Tree-Structured Rollout Approach for Multi-Turn LLM Agents},
  author  = {Zhang, Yijun and Xu, Fan and Ding, Jiaxin and Xie, Yule and Gao, Shiqing and Ding, Xin and Zhang, Haoxiang and Fu, Luoyi and Wang, Xinbing},
  journal = {arXiv preprint arXiv:2607.06223},
  year    = {2026}
}

@inproceedings{xiong2026scaling,
  title={Scaling search-augmented llm reasoning via adaptive information control},
  author={Xiong, Siheng and Gungordu, Oguzhan and Johnson, Blair and Kerce, James Clayton and Fekri, Faramarz},
  booktitle={The 1st Workshop on Scaling Post-training for LLMs}
}

@inproceedings{xiong2026enhancing,
  title={Enhancing language model reasoning with structured multi-level modeling},
  author={Xiong, Siheng and Payani, Ali and Fekri, Faramarz},
  booktitle={International Conference on Learning Representations},
  volume={2026},
  pages={36557--36610},
  year={2026}
}

@article{xiong2025enhancing,
  title={Enhancing long chain-of-thought reasoning through multi-path plan aggregation},
  author={Xiong, Siheng and Payani, Ali and Fekri, Faramarz},
  journal={arXiv preprint arXiv:2510.11620},
  year={2025}
}
\bibliographystyle{iclr2026_conference}

\newpage
\appendix

\startcontents[appendix]
\section*{Appendix Content Table}

\printcontents[appendix]{}{1}{\setcounter{tocdepth}{3}}

\newpage

\section{Related Work}
Large language models (LLMs) have advanced rapidly in recent years \cite{ye2025llms4all, xiong2026scaling, zhang2026mapro, xiong2026enhancing, xiong2025enhancing, zhang2025agentrouter}. Building on this progress, LLM-driven agents and agentic reinforcement learning have gained prominence for their ability to plan, interact, and solve complex tasks with limited human oversight \cite{zhang2026semantic, shi2026sage, shi2026ng, huang2026evolverouter, bao2026drift}. A central practical challenge accompanying these successes is that RL for LLM reasoning and tool-using agents faces a core inefficiency: rollout budgets are typically allocated uniformly, despite large variation in difficulty and informativeness. Early work addresses this at the \emph{prompt level} by allocating more rollouts to harder or more informative problems, using success rate \citep{yang2025depth, xiong2025reinforce, zhao2026training}, capability \citep{yao2026coba, panaganti2026group}, or reward variance \citep{nguyen2026adaptive, fang2026allocate} as proxies. While effective, these methods rely on final outcomes and thus overlook where uncertainty arises within multi-step reasoning. This issue is amplified in agentic too use settings, where rewards are sparse and delayed, and failures often originate from specific intermediate decisions, making prompt-level signals insufficient.

To capture finer-grained structure, recent work shifts to \emph{step-level allocation}. However, direct reward signals are harder to retrieve at within steps, motivating the use of \emph{internal signals} such as \emph{Uncertainty}. For example, entropy-based methods identify locally uncertain steps and trigger branching accordingly \citep{zheng2025first, wei2026entropy}, while variants further compare uncertainty across steps or between question and tool states to decide whether to expand exploration \citep{dong2025agentic, dong2026agentic}. Yet entropy only reflects local ambiguity. To distinguish meaningful from spurious uncertainty, mutual-information-based methods measure dependency between reasoning steps and inputs, capturing whether a branch is informative or merely noisy \citep{wang2026ragen, duan2025uprop}.

Recent concurrent work further learns allocation signals at the state or prefix level. TRACE~\citep{zou2026trace} predicts conditional success from prefix histories and allocates visits to promising intermediate anchors. 3SPO~\citep{han2026three} derives state scores from historical success statistics and couples adaptive rollout allocation with step-wise credit assignment and post-step updates. IGRPO~\citep{zhang2026igrpo} instead expands trajectory trees according to information gain and derives an induced teacher distribution for policy optimization. RAIL differs by learning the realized marginal gain of structured interventions that jointly vary branch budget and decoding behavior, while retaining the underlying group-relative policy objective.

Dynamic rollout allocation also interacts with \emph{policy optimization}. Existing approaches modify token-level advantages~\citep{fang2026allocate,dong2026agentic}, introduce step-level credit assignment~\citep{dong2025agentic,yang2025depth,han2026three}, filter low-signal trajectories~\citep{wang2026ragen}, or construct tree-derived rewards and policy targets~\citep{ji2025tree,zhao2026training,zhang2026igrpo}. These methods often modify rollout collection and the optimization rule jointly. In contrast, RAIL isolates intervention learning at the rollout-generation layer: it retains the critic-free group-relative objective while learning where and how to reshape the rollout distribution from realized intervention outcomes.

\section{Implementation Details and Additional Experiments}

\subsection{Benchmarks}
\label{app:benchmark}

\noindent\textbf{AgentBench-OS}~\citep{liu2024agentbench} is an operating-system interaction benchmark from AgentBench. The agent receives a natural-language instruction and completes the task by issuing shell commands in an interactive environment. Unlike single-turn reasoning benchmarks, the outcome depends on a sequence of executable decisions, and early commands may alter the environment state in ways that make later recovery easier or harder. This benchmark is therefore suitable for evaluating whether RAIL can identify trajectory states where additional rollout intervention exposes more useful recovery paths. In our paper, We train on 600 bash tasks and evaluate on the full test held-out test tasks that carry a precomputed ground-truth answer. Each episode runs in its own pooled \texttt{ubuntu:20.04} container (reset between tasks), and every command is executed under a $5$-second hard timeout with truncated output. A task is judged successful either by running its AgentBench checker scripts in the container (all must exit $0$) or, when a reference answer is available, by a case-insensitive match with integer-value and bounded-substring tolerance.

\noindent\textbf{AgentBench-DB}~\citep{liu2024agentbench} is a database interaction benchmark from AgentBench. The agent must reason over structured database states and execute operations that satisfy the task instruction. We report both overall success rate and operation-level results over \textsc{Select}, \textsc{Insert}, and \textsc{Update} tasks. This benchmark tests rollout intervention in a structured tool-use setting, where failures may arise from incorrect query construction, wrong operation choice, or insufficient use of intermediate database feedback. In this paper, we train on 1000 tasks and evaluate on the $300$-task test split, which is balanced across operation types ($100$ \textsc{Insert}, $100$ \textsc{Update}, and $100$ \textsc{Select}-family tasks). The agent uses two tools, \texttt{sql\_query(query)} and \texttt{answer\_action(answer)}, against a fresh \texttt{MySQL~8} instance whose tables are (re)created and populated per episode. \textsc{Select} tasks are graded by a format-tolerant comparison of the returned values (numeric canonicalisation and order-independent set equality); write tasks are graded by \emph{database state}: after the model's writes, the resulting table state is hashed and compared against the state produced by executing the reference SQL on a fresh copy, so any SQL with the correct effect scores as correct regardless of surface form.

\noindent\textbf{WebShop}~\citep{yao2022webshop} is a web-based shopping benchmark in which an agent must satisfy a user request by navigating product search pages and selecting an appropriate item. The agent interacts with the environment through search, page navigation, item inspection, and purchase decisions. We use WebShop to evaluate RAIL under web-navigation rollouts, where the agent must explore multiple candidate products and recover from misleading or incomplete search trajectories. We report overall success rate as well as difficulty-stratified results over L1 and L2 tasks following the difficulty setting from \cite{ma2024agentboard}.
In this paper, we use the 600-product WebShop index, which exposes $6910$ human-instruction goals; goals are partitioned into disjoint index ranges, and we use the full held-out test set as evaluation. The agent has three tools, \texttt{search\_action(query)}, \texttt{click\_action(button)}, and \texttt{answer\_action(answer)}, over at most $10$ rounds, and a purchase is realised by clicking \texttt{Buy Now}, which terminates the episode. 

\noindent\textbf{ToolQA-Coffee}~\citep{zhuang2023toolqa} is a tool-mediated question answering benchmark over coffee-domain tabular data. In our configuration, the agent uses a Python interpreter to inspect tables, perform lookup and aggregation, and answer questions that require data analysis rather than memorized knowledge. We evaluate both Easy and Hard splits to measure whether rollout intervention remains useful as the required tool-use and reasoning complexity increases.
In our paper, all questions are templated over a single coffee price time series (daily open/high/low/close/volume records). We use a merged split of 600 training for Coffee-hard and 200 test questions. By the official setting, easy questions are single-day lookups and derivations (opening/closing price, percentage change, daily range, bullish/bearish), while hard questions are multi-day aggregations over a date window (extremal price, price range, date of the largest day-over-day move). The agent has two tools, \texttt{python\_interpreter(code)} and \texttt{answer\_action(answer)}, over at most $6$ rounds; the interpreter runs in a sandboxed worker with a restricted import set and a per-call fresh copy of the data frame. Answers are scored by ToolQA's normalised exact match, augmented with a relative/absolute numeric tolerance so that equivalent numeric formattings (e.g.\ \texttt{147} vs.\ \texttt{147.0}) are accepted.

\begin{table}[t]
\centering
\small
\label{tab:cross-dimension}
\resizebox{0.9\textwidth}{!}{
  \begin{tabular}{@{}cccccc@{}}
  \toprule
  \textbf{Method} & \textbf{Granularity} & \textbf{Calibrated?} & \textbf{Learns When?} & \textbf{Adaptive Strategy?} & \textbf{Strategies} \\
  
  \midrule\midrule
  
    \textbf{GRPO}  & Task-level  & \xmark & \xmark &
  \xmark & None\\
  \midrule

   ARPO & Step-level & \xmark & \xmark &
  \xmark & Branch Count \\
   VIP & Task-level & \cmark & \cmark &
  \xmark & Branch Count\\
    AEPO & Step-level & \xmark & \xmark &
  \xmark & Branch Count \\
   Tree-GRPO & Step-level & \xmark & \xmark &
  \xmark & None\\
  TAMPO & Task-level & \xmark & \xmark &
  \cmark & Temperature\\
  
  \midrule

  \rowcolor{blue!5} \textbf{RAIL}          & Step-level & \cmark & \cmark &
  \cmark & Multi-Strategy \\

  \midrule
  \bottomrule
  \end{tabular}}
    \vspace{-5pt}
  \caption{Cross-dimensional comparison of rollout intervention baselines.}
\end{table}

\subsection{Baselines}
\label{app:baselines}
\noindent\textbf{GRPO}~\citep{shao2024deepseekmath} serves as the base policy optimization framework without adaptive intervention. It performs group-based updates by normalizing rewards across a fixed number of sampled rollouts, implicitly assuming all rollouts contribute equally regardless of their variability or informativeness.

\noindent\textbf{Step-level Branching. }Step-level branching methods, including ARPO~\citep{dong2025agentic} and AEPO~\citep{dong2026agentic}, spawn additional rollouts from intermediate states whose token entropy is high. ARPO tracks the entropy change relative to an initial anchor and branches once it exceeds a threshold, while AEPO splits the global and branch budget through a sigmoid of the question-versus-tool entropy gap and penalizes lineages that branch repeatedly. However, the branch trigger is a fixed function of an output-distribution uncertainty signal, which is decoupled from the environment recoverability that determines whether extra rollouts can restore a useful learning signal.

\noindent\textbf{Task-level Predictive Budget Allocation. }These methods adapt a per-task scalar from a predicted or observed reward signal. VIP~\citep{nguyen2026adaptive} relates a group's gradient variance to a Bernoulli success model and solves a budget-constrained program that assigns more rollouts to the tasks predicted to be most variance-rich, whereas TAMPO~\citep{dang2026temperature} treats the decoding temperature as a reward-driven meta-policy, sampling hotter or cooler decoding where higher reward has been observed. However, both commit to a single scalar per task, how many rollouts, or how exploratory the decoding, and cannot express which kind of intervention a specific mid-trajectory state requires.

\noindent\textbf{Tree-structured Rollouts. }Tree-GRPO~\citep{ji2025tree} collects rollouts as prefix-sharing trees rather than independent trajectories, computing an intra-tree advantage from the post-branch suffixes and an inter-tree advantage across trees, so that more rollouts fit a fixed budget and a step-level signal is derived from the shared prefixes. However, the tree structure is fixed a priori rather than chosen per state, so compute is not redirected toward the specific decisions whose additional rollouts would most improve informativeness.

\subsection{Framework Details}
\label{app:rail_details}

\begin{algorithm}[t]
\caption{\textsc{RAIL}: Recoverability-Aware Intervention Learning}
\label{alg:rail}
\small
\begin{algorithmic}[1]

\Require Policy $\pi_{\theta}$, structured intervention space
$\mathcal{B}\!=\!\mathcal{M}\times\mathcal{T}$, recoverability controller
$\widehat{\Delta}_{\phi}$, utility threshold $\eta$
\Ensure Updated policy $\pi_{\theta}$ and controller
$\widehat{\Delta}_{\phi}$

\State Initialize intervention buffer $\mathcal{D}\gets\varnothing$
\State Initialize mode $q\gets\textsc{Shadow}$

\For{each policy-optimization step $t$}

    \State Generate an initial rollout group $Y_t$ using $\pi_{\theta_t}$
    \State Identify candidate anchors $\mathcal{A}_t$ and construct states $z$

    \If{$q=\textsc{Shadow}$}

        \ShadowHeader{Shadow Mode: Heuristic Trace Collection}

        \ForAll{$z\in\mathcal{A}_t$}
            \ForAll{$b\in\mathcal{B}$ in iterative order}

                \State Execute intervention $b$ and observe
                $
                    y
                    =
                    I(Y_{\mathrm{after}})
                    -
                    I(Y_{\mathrm{before}})
                    -
                    \lambda C(b)
                $
                \State $\mathcal{D}\gets\mathcal{D}\cup\{(z,b,y)\}$
                \State Update $\widehat{\Delta}_{\phi}$ from recent traces in $\mathcal{D}$

                \If{$b$ improves the rollout-group signal}
                    \State Retain the generated continuations
                \Else
                    \State \textbf{break}
                    \Comment{accept-or-stop}
                \EndIf

            \EndFor
        \EndFor

        \If{the promotion criterion (fixed or dynamic) is satisfied}
            \State $q\gets\textsc{Live}$
        \EndIf

    \Else

        \LiveHeader{Live Mode: Learned Utility-Gated Intervention}

        \ForAll{$z\in\mathcal{A}_t$}

            \State Select $b_t^\star
                \in
                \arg\max_{b\in\mathcal{B}}
                \widehat{\Delta}_{\phi_t}(z,b)
            $

            \If{$\widehat{\Delta}_{\phi_t}(z,b_t^\star)>\eta$}
                \State Execute $b_t^\star$
                \Comment{with exploration during training}
                \State Observe $y_t$ and retain the generated continuations
                \State $\mathcal{D}\gets
                \mathcal{D}\cup\{(z,b_t^\star,y_t)\}$
                \State Update $\widehat{\Delta}_{\phi}$ from recent traces
            \Else
                \State Continue the default rollout process
            \EndIf

        \EndFor

    \EndIf

    \State Compute rewards and relative advantages from $Y_t$
    \State Update $\pi_{\theta_t}\rightarrow\pi_{\theta_{t+1}}$
    using the standard critic-free objective

\EndFor

\State \Return $\pi_{\theta}$ and $\widehat{\Delta}_{\phi}$

\end{algorithmic}
\end{algorithm}

This appendix instantiates the abstract controller interface used in Section~\ref{sec:method}. The main method requires the controller to compare interventions through $\widehat{\Delta}_{\phi}(z,b)$; therefore, implementation must specify how actions $b$ are encoded, what pre-intervention information forms the state $z$, how realized gains $y$ are computed, and how the live utility gate is applied. In this section we describe RAIL algorithm in details. 

\subsubsection{Structured Intervention Actions}
RAIL uses a compact executable intervention space that varies both rollout quantity and rollout behavior:
\[
    \mathcal{B}
    =
    \left(\mathcal{M}\times\mathcal{T}\right),
    \qquad
    \mathcal{M}=\{4,8,12\},
    \qquad
    \mathcal{T}=\{\textsc{Exploit},\textsc{Mild},\textsc{Aggressive}\}.
\]
Here, each $m\in\mathcal{M}$ denotes the number of additional continuations spawned at an anchor, and each $\tau\in\mathcal{T}$ specifies a decoding regime. For AgentBench-OS and AgentBench-DB, the regimes are
\[
\textsc{Exploit}=(T_{\mathrm{base}},0.90),\quad
\textsc{Mild}=(1.3,0.98),\quad
\textsc{Aggressive}=(1.6,1.0),
\]
where each pair denotes temperature and top-$p$, and $T_{\mathrm{base}}$ inherits the baseline sampling temperature. For WebShop and ToolQA-Coffee, we use more exploratory regimes, with \textsc{Mild}=$(1.8,0.99)$ and \textsc{Aggressive}=$(2.5,1.0)$, because these environments require broader search over web or table-operation trajectories.

This $3\times3$ space gives nine executable branch actions. It is intentionally small enough for online partial-feedback learning, but not tied to these exact values. \textbf{The same interface can be extended} by adding another scalar choice, such as prompting style, verifier-guided repair, or rollback actions; or by adding option choices in one scalar, such as what we did in ablation study, where we vary the granularity of the discrete grid: a coarser $2\times2$ space (branch counts ${4,12}$, regimes ${\textsc{exploit},\textsc{aggressive}}$) and a finer $4\times4$ space that introduces an additional branch budget ($m=16$) and a fourth intermediate temperature regime. In the controller input, an action is encoded by a four-dimensional vector: one normalized branch-count scalar and a three-dimensional one-hot code for the
decoding regime. For the default \(3\times3\) action space (branch counts 4, 8, 12), the branch count is
normalized as \(m/12\). For action-granularity ablations with different maximum
branch budgets, we normalize \(m\) by the maximum branch budget in the
corresponding action grid.
The scalar budget encoding lets nearby branch counts share statistical strength, while the one-hot regime encoding avoids imposing an artificial ordering over qualitatively different sampling behaviors.

\subsubsection{State Representation}
The controller must predict intervention gain before the branch outcome is observed, so $z_t$ is restricted to leakage-safe pre-intervention information. We use a 26-dimensional state vector:
\[
    z_t
    =
    \left[
    \underbrace{s_t}_{10\ \mathrm{scalars}}
    \,\Vert\,
    \underbrace{P_h h_t}_{8\ \mathrm{dims}}
    \,\Vert\,
    \underbrace{P_e e_x}_{8\ \mathrm{dims}}
    \right],
    \qquad
    \dim(z_t)=26.
\]
The first block $s_t$ contains ten rollout statistics. They are chosen to capture four aspects that determine recoverability from empirical statistical analysis: local uncertainty, trajectory progress, current reward-distribution shape, and diversity of the existing rollout pool. The exact features are listed in Table~\ref{tab:rail_state_features}. All statistics are computed from the rollout pool before the attempted branch, together with anchor metadata and budget counts; the after-branch outcome is never used as an input feature.

\begin{table}[t]
\centering
\small
\begin{tabular}{l p{0.63\linewidth}}
\toprule
\textbf{Feature} & \textbf{Role in recoverability prediction} \\
\midrule
Anchor entropy & Measures local decision uncertainty; high entropy indicates that the policy has not collapsed to a single continuation mode. \\
Anchor round & Locates the intervention point in the trajectory; early mistakes often create larger downstream recovery effects. \\
Anchor round fraction & Normalizes position by the maximum number of rounds, making trajectory progress comparable across tasks. \\
Budget used fraction & Indicates remaining rollout capacity, which affects whether further branching is worth its cost. \\
Before-branch reward mean & Captures current task solvability under the existing pool. \\
Before-branch reward standard deviation & Measures existing reward contrast before intervention. \\
Distance to $0.5$ & Estimates the remaining gain ceiling; pools near all-success or all-failure have less relative-reward contrast. \\
Failed-anchor count & Tracks repeated unsuccessful intervention attempts, signaling potentially low recoverability for the current task. \\
Average pool entropy & Measures uncertainty across the rollout pool, not only at the selected anchor. \\
Tool-sequence diversity & Measures behavioral diversity by counting distinct tool-call sequences relative to pool size. \\
\bottomrule
\end{tabular}
\caption{Pre-intervention scalar features used by the RAIL controller.}
\label{tab:rail_state_features}
\end{table}

The last two scalar features, average pool entropy and tool-sequence diversity, were included after offline feature selection: average pool entropy produced the largest act/skip-AUC improvement, and tool-sequence diversity gave an additional gain. A tested visitation-uncertainty feature was removed because its signal largely reflected recurrence artifacts rather than recoverability. The projected blocks provide lightweight context beyond hand-designed statistics: $h_t$ is the policy hidden state at the anchor and $e_x$ is the task embedding. Both are projected to eight dimensions, which keeps the controller cheap and the projection
matrices \(P_h\) and \(P_e\) are trained with the controller, while the policy
hidden states are detached so that controller training does not update the
underlying policy model.

\subsubsection{Controller Architecture and Training}
Given the 26-dimensional state and 4-dimensional action code, the controller predicts a scalar gain from a 30-dimensional state-action input. We use a single-output neural network rather than a multi-head action classifier, because each trace observes the outcome of only one selected action. This design matches the contextual-bandit feedback structure: the controller learns a utility function over $(z,b)$ pairs and generalizes across nearby states and actions. The default network uses a 64-dimensional hidden layer with GELU activation, dropout, and a skip connection from the raw state-action input to the output layer. \textbf{The controller runs separately from the policy model; hidden states are detached before entering the controller, so controller gradients never update the language model.}

Each executed branch produces a trace $(z_t,b_t,y_t)$. The controller is trained on a replay buffer with recency-weighted Huber regression:
\[
    \mathcal{L}(\phi)
    =
    \frac{
    \sum_j w_j
    \ell_{\kappa}
    \left(
        \widehat{\Delta}_{\phi}(z_j,b_j)-y_j
    \right)}
    {\sum_j w_j}.
\]
We use Huber loss because finite rollout groups produce noisy gain labels, and we use recency weighting because recoverability changes as the policy evolves. In implementation, the hard window in \eqref{eq:rail_predictor_obj} is realized as an exponential half-life over recent traces.

\textbf{Recoverability label and intervention cost.}
Our benchmarks are all trained with binary reward, despite benchmarks such as WebShop provide dense reward options. The implementation computes recoverability gain through the reduction in distance to the balanced reward regime:
\[
    y_t
    =
    \left[
        d(Y_{\mathrm{before}})
        -
        d(Y_{\mathrm{after}})
    \right]
    -
    \lambda C(b_t),
    \qquad
    d(Y)=|\bar R_Y-0.5|.
\]
This label is monotone with the reward-contrast objective in the binary case: moving the rollout pool toward $\bar R_Y=0.5$ increases the variance term $p(1-p)$ and creates stronger relative-reward signal. In practice, a cost penalty is included in all reported runs with $\lambda=0.005$. It prevents the controller from learning a trivial preference for larger branch budgets and makes the predicted gain correspond to utility after accounting for rollout cost.

\subsubsection{Shadow-mode Execution}
The main text describes shadow mode conceptually; here we specify its execution rule. RAIL begins from an initial pool of four rollouts and detects candidate anchors using normalized token entropy. At each anchor, the warm-up policy tries branch budgets in increasing order, $4\rightarrow8\rightarrow12$, clamped by the remaining rollout budget (Rollout budget is the maximum budget allowed for branching and allocation.) A branch is retained if the pooled reward distribution moves closer to the balanced regime; otherwise, the sweep at that anchor stops. Each accepted or rejected trial is logged as a separate trace, so one anchor can provide multiple supervised outcomes for the controller. The decoding regime is chosen from the current pool mean: easier pools use \textsc{Exploit}, intermediate pools use \textsc{Mild}, and harder or unresolved pools use \textsc{Aggressive}; when adaptive decoding is enabled, this regime is recomputed before each budget increment.

\begin{wrapfigure}{t}{0.6\textwidth}

\centering
\includegraphics[width=\linewidth]{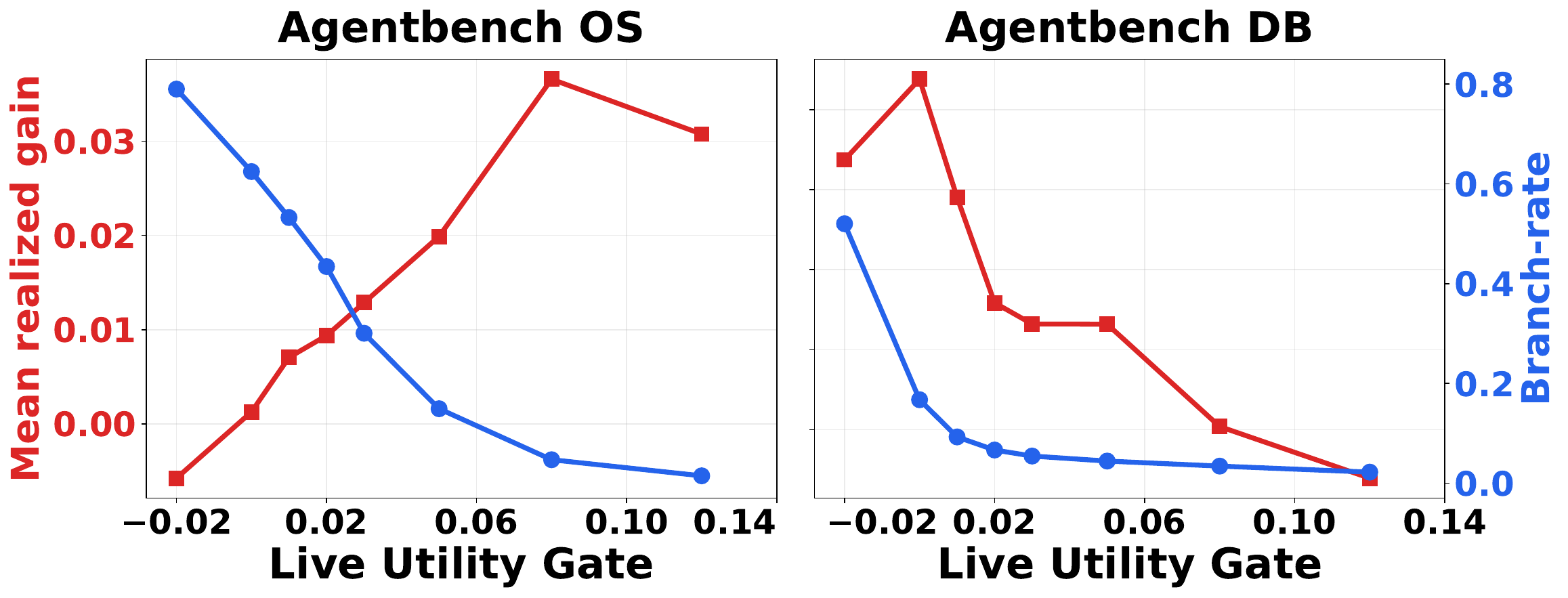}
\vspace{-15pt}
\caption{Live Utility Gate Trade-off against realized gain on AgentBench-OS (Left) and AgentBench-DB (Right). }
\label{fig:tau}
\vspace{-10pt}
\end{wrapfigure}

\subsubsection{Shadow-To-Live Promotion}
The shadow phase is only needed to bootstrap the recoverability controller before it is trusted to control live rollout allocation. In practice, we choose a fixed proportion of the total steps to serve as the shadow length for simplicity. We demonstrate how the change affects the training in Table~\ref{tab:ablation} and discussed its trade-offs in the experiment sections. However, a fixed gate is nevertheless a coarse choice: if too short, the controller may enter live mode before its predictions are decision-relevant; if too long, training continues to spend rollout budget on heuristic trace collection after the controller is already useful. \textbf{Since controller reliability depends on both the benchmark and the evolving policy, we also evaluate a dynamic promotion rule that switches from shadow to live mode once the controller demonstrates sustained out-of-sample predictive skill.}

We measure this skill prequentially: before each controller update, we record the prediction made for the executed intervention and compare it with the realized gain after the outcome is observed. This predict-before-update protocol prevents the promotion signal from being inflated by fitting the same trace. We use sign agreement rather than regression error,
\[
    a_t
    =
    \mathbf{1}
    \left[
    \operatorname{sign}
    \left(
    \widehat{\Delta}_{\phi_t}(z_t,b_t)
    \right)
    =
    \operatorname{sign}(y_t)
    \right],
\]
because live deployment primarily depends on whether an intervention is predicted to improve over the no-op baseline. In contrast, absolute regression error can be misleading when realized gains are concentrated near small negative values: a nearly constant predictor may achieve low MAE while providing little useful act/skip information.

Let $\bar a_t$ denote the mean sign agreement over the most recent $W$ intervened anchors. The dynamic rule promotes the controller from shadow to live mode when three conditions hold simultaneously: (i) at least $n_{\min}$ anchors have been observed; (ii) all nine executable branch actions have been observed at least once (During shadow mode, a small coverage component occasionally samples less-used ($m,\tau$) cells so that all executable actions can be observed before live
promotion); and (iii) $\bar a_t$ exceeds a chance-adjusted threshold for $P$ consecutive anchors. Specifically, we use
\[
    \beta_{\mathrm{bar}}
    =
    0.5
    +
    2\sqrt{\frac{0.25}{W}},
\]
which is two binomial standard deviations above chance agreement under a balanced-sign null model. Promotion is latched once triggered. As a safety fallback, if the dynamic gate has not fired by a maximum step cap $T_{\mathrm{cap}}$, the controller is promoted by the cap; conversely, if the controller never demonstrates stable predictive skill before the end of training, it remains in shadow mode rather than deploying an unreliable predictor.

In the dynamic-promotion ablation, we set the sign-agreement window to $W=60$ intervened anchors, the consecutive-success requirement to $P=60$ anchors, the minimum number of observed anchors to $n_{\min}=50$, and the safety step cap to $T_{\mathrm{cap}}=200$ rollout steps, which gives a chance-adjusted promotion threshold of $\beta_{\mathrm{bar}}\approx0.629$. Here $W$ is the trailing window over which the mean sign agreement $\bar a_t$ is computed, $P$ is how many consecutive anchors $\bar a_t$ must remain above $\beta_{\mathrm{bar}}$ before promotion, $n_{\min}$ is the warm-up floor of anchors that must be seen first, and $T_{\mathrm{cap}}$ is the fallback step at which the controller is promoted regardless. Under this rule, the controller promotes by demonstrated skill rather than by the cap: on AgentBench-OS it enters live mode at step $101$ with rolling sign agreement $0.650$, and on AgentBench-DB at step $165$ with rolling sign agreement $0.783$—both well before $T_{\mathrm{cap}}$. Thus, dynamic promotion adapts the shadow length to the controller's observed reliability, reaching live intervention earlier when the controller is ready while avoiding premature deployment when its predictions are not yet stable.

\begin{figure}[t]
	\centering
    \vspace{-10pt}
	\includegraphics[width=1\linewidth]{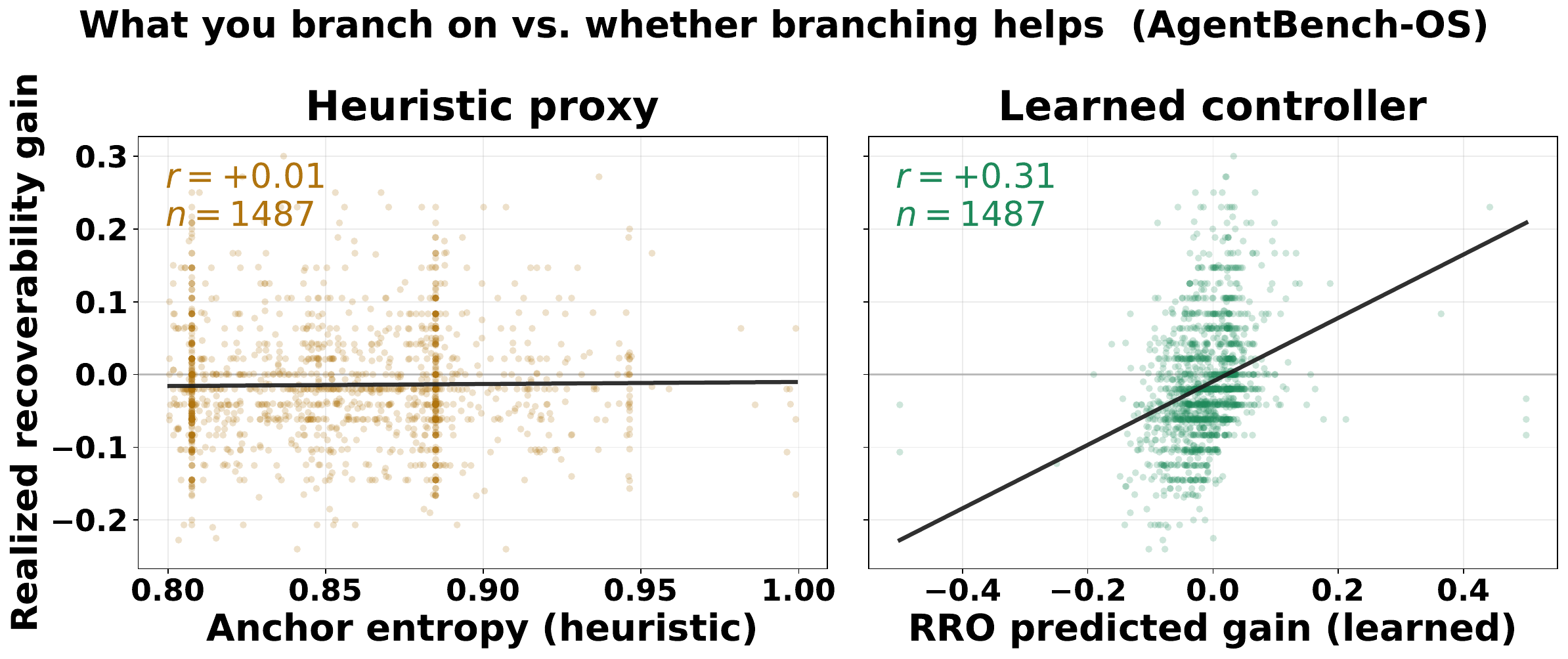}
        \vspace{-10pt}
	\caption{Recoverability is not readable from entropy, but it is learnable.
    }

    \label{fig:revocer}
\end{figure}

\subsubsection{Live Utility Gate}
In live mode, the controller scores the nine executable branch actions and compares the best predicted branch against the no-op baseline. The no-op has fixed gain zero, so the live decision rule is
\[
    b_t^\star
    =
    \arg\max_{b\in\mathcal{B}\setminus\{\varnothing\}}
    \widehat{\Delta}_{\phi_t}(z_t,b),
    \qquad
    b_t
    =
    \begin{cases}
    b_t^\star, &
    \widehat{\Delta}_{\phi_t}(z_t,b_t^\star)>\eta_{\mathrm{gain}},\\
    \varnothing, &
    \text{otherwise}.
    \end{cases}
\]
Thus, $\eta_{\mathrm{gain}}$ is a margin over the no-op baseline: larger values make RAIL more selective, while smaller values increase branching frequency. We use $\eta_{\mathrm{gain}}=0.02$ for our experiments. In practice we recommend tuning the threshold within a small range because the reward scale and anchor frequency differ across these environments. Here we provide another ablation study on how the change of $\eta_{\mathrm{gain}}$ affect overall recoverability gain in Figure~\ref{fig:tau}. Unlike shadow mode, live mode executes a single selected branch action rather than running an accept-or-stop sweep. 

\textbf{It's worth noting that Recoverability is not readable from entropy, but it is learnable.} Figure~\ref{fig:revocer} probes why rollout intervention must be learned rather than triggered by a fixed uncertainty signal. Both panels are computed on the same set of live-phase decision states logged during RAIL training on AgentBench-OS ($n=1487$ anchors), and both plot, on the vertical axis, the realized recoverability gain — the change in finite-group reward contrast actually observed after the intervention. The left panel places on the horizontal axis the per-anchor decision entropy, the local-uncertainty signal on which step-level branchers such as ARPO and AEPO base their intervention decisions; it is essentially uncorrelated with realized gain (Pearson $r=+0.01$), indicating that knowing a state is high-entropy tells one almost nothing about whether intervening there will expose additional learning signal. The right panel places on the horizontal axis the gain predicted by RAIL's recoverability controller, which is trained online from realized outcomes; its predictions track the realized gain far more closely ($r=+0.31$). 

%\subsubsection{Exploration and Promotion.}
\subsubsection{Live Exploration} 
To maintain coverage of the finite intervention space and guarantee sufficient exploration, live training uses a small exploration probability over executable branch cells. Exploration is directed toward under-covered cells rather than sampled uniformly, which improves coverage of the $3\times3$ action space under limited live-mode budget. The exploration rate decays over training. We discuss how RAIL enforces sufficient exploration in Appendix~\ref{app:regret}.

\subsubsection{Time complexity Analysis} 
While wall-clock time is difficult to compare fairly across implementations because it depends on hardware, serving configuration, batching, parallelism, environment latency, and system load, we instead analyze computational complexity. RAIL leaves the asymptotic cost of the training loop unchanged relative to GRPO. The dominant term for both is language-model rollout generation and environment interaction, which scales as $O(G \cdot L \cdot E)$ for $G$ rollouts of $L$ generated tokens over $E$ environment steps; every other component is lower-order. RAIL's only per-step addition is the recoverability controller, a small side-predictor evaluated at candidate anchor states. Scoring an anchor is a single forward pass of a two-layer MLP ($\sim$2{,}000 parameters) over the $|\mathcal{B}|$ executable interventions, which is negligible compared with LLM rollout generation. Consequently, the controller contributes no measurable overhead: RAIL's compute is governed by how many rollouts it generates, not by the intervention machinery. In summary, since the utility gate suppresses branching at states with low predicted recoverability, RAIL in fact issues fewer rollouts per task on average than uniform GRPO-32 (Table~\ref{tab:efficiency}), and it's live time branching cost is comparable to other methods such as ARPO or AEPO despite the added controller.

\section{Theory and Proof}
\subsection{Reward Variance and the Midpoint Recoverability Target}
\label{app:variance_signal}

The recoverability target in this paper is instantiated as an increase in reward variance, which is an indirect proxy for optimization signal quality. This appendix justifies why RAIL uses reward variance as the recoverability signal and why the midpoint expected reward, i.e., success probability $p=0.5$, is the natural target for informative rollout intervention. The argument is closely related to the variance analysis of VIP~\citep{nguyen2026adaptive}, which shows that in group-based RL with verifiable rewards, the prompt-dependent factor in the projected-gradient variance is governed by the binary reward-variance term $p(1-p)$. We adapt this insight from prompt-level rollout allocation to our state- and intervention-level setting: rather than allocating more rollouts to prompts with high predicted variance, RAIL asks whether an intervention increases the reward-distribution contrast available at a task or trajectory state.

\textbf{Setup.}
Fix a rollout state $z$ and an intervention $b\in\mathcal{B}$. Let $R_b(z)\in\{0,1\}$ denote the verifier reward of a rollout generated after applying intervention $b$. We define
\begin{equation}
\label{eq:reward_variance_setup}
    p_b(z)=\Pr(R_b(z)=1),
    \qquad
    \mu_b(z)=\mathbb{E}[R_b(z)],
    \qquad
    V_b(z)=\operatorname{Var}(R_b(z)).
\end{equation}
For a group of $n$ conditionally i.i.d. rollouts, let $R_1,\ldots,R_n\sim R_b(z)$ and $\bar R=\frac{1}{n}\sum_{i=1}^n R_i$. Group-based methods form reward-dependent relative advantages, for example
\begin{equation}
\label{eq:centered_loo_adv}
    A_i^{\mathrm{cent}} = R_i-\bar R,
    \qquad
    A_i^{\mathrm{loo}}
    =
    R_i-\frac{1}{n-1}\sum_{j\neq i}R_j .
\end{equation}
Thus, when all rewards in a rollout group are identical, the reward-dependent advantage signal degenerates. Nontrivial reward contrast arises only when the rollout group contains different reward outcomes.

\begin{lemma}[Reward variance controls relative advantage contrast]
\label{lem:advantage_contrast}
For conditionally i.i.d. rewards with variance $\sigma_R^2=\operatorname{Var}(R_b(z))$, the expected mean squared relative advantages satisfy
\begin{equation}
\label{eq:advantage_contrast_both}
    \mathbb{E}\!\left[
    \frac{1}{n}\sum_{i=1}^n
    \left(A_i^{\mathrm{cent}}\right)^2
    \right]
    =
    \frac{n-1}{n}\sigma_R^2,
    \qquad
    \mathbb{E}\!\left[
    \frac{1}{n}\sum_{i=1}^n
    \left(A_i^{\mathrm{loo}}\right)^2
    \right]
    =
    \frac{n}{n-1}\sigma_R^2 .
\end{equation}
\end{lemma}

\begin{proof}
For the centered estimator,
\begin{equation}
\label{eq:centered_proof_identity}
    \sum_{i=1}^n (R_i-\bar R)^2
    =
    \sum_{i=1}^n R_i^2 - n\bar R^2 .
\end{equation}
Let $\mu=\mathbb{E}[R_b(z)]$. Since
$\mathbb{E}[R_i^2]=\sigma_R^2+\mu^2$ and
$\mathbb{E}[\bar R^2]=\operatorname{Var}(\bar R)+\mu^2=\sigma_R^2/n+\mu^2$, we have
\begin{equation}
\label{eq:centered_proof_expectation}
\begin{aligned}
    \mathbb{E}\!\left[\sum_{i=1}^n (R_i-\bar R)^2\right]
    &=
    n(\sigma_R^2+\mu^2)
    -
    n\left(\frac{\sigma_R^2}{n}+\mu^2\right)  \\
    &=
    (n-1)\sigma_R^2 .
\end{aligned}
\end{equation}
Dividing by $n$ gives the first equality in~\eqref{eq:advantage_contrast_both}. For the leave-one-out estimator,
\begin{equation}
\label{eq:loo_centered_relation}
    A_i^{\mathrm{loo}}
    =
    R_i-\frac{n\bar R-R_i}{n-1}
    =
    \frac{n}{n-1}(R_i-\bar R).
\end{equation}
Substituting~\eqref{eq:loo_centered_relation} into the centered result gives the second equality in~\eqref{eq:advantage_contrast_both}.
\end{proof}

\textbf{Connection to VIP's projected-gradient analysis.}
Lemma~\ref{lem:advantage_contrast} gives a direct advantage-level justification for reward variance. VIP provides a complementary gradient-level justification: under standard conditional i.i.d. rollout assumptions and second-order decorrelation assumptions between rewards and projected score terms, the per-prompt projected-gradient variance of Dr.~GRPO and RLOO is proportional to the same reward-variance factor~\citep{nguyen2026adaptive}. In the notation of VIP, for binary rewards encoded as $\{-1,+1\}$ with success probability $p$, the reward-dependent term is $4p(1-p)$; for the $\{0,1\}$ encoding used here, the corresponding term is $p(1-p)$. The two encodings differ only by a constant factor and therefore have the same maximizer. This establishes that reward variance is not merely a heuristic measure of diversity: it is the reward-dependent contrast term that governs relative-advantage signal, and appears explicitly in the projected-gradient variance analysis of group-based RL.

\begin{proposition}[Midpoint maximizes binary reward contrast]
\label{prop:midpoint_reward_variance}
For verifier rewards $R_b(z)\in\{0,1\}$,
\begin{equation}
\label{eq:bernoulli_variance}
    V_b(z)
    =
    \operatorname{Var}(R_b(z))
    =
    p_b(z)(1-p_b(z)).
\end{equation}
This quantity is uniquely maximized at $p_b(z)=1/2$, equivalently $\mathbb{E}[R_b(z)]=0.5$.
\end{proposition}

\begin{proof}
Since $R_b(z)$ is Bernoulli with success probability $p_b(z)$,
\begin{equation}
\label{eq:bernoulli_variance_derivation}
    \operatorname{Var}(R_b(z))
    =
    \mathbb{E}[R_b(z)^2]-\mathbb{E}[R_b(z)]^2
    =
    p_b(z)-p_b(z)^2
    =
    p_b(z)(1-p_b(z)).
\end{equation}
Moreover,
\begin{equation}
\label{eq:bernoulli_derivatives}
    \frac{d}{dp_b}p_b(1-p_b)=1-2p_b,
    \qquad
    \frac{d^2}{dp_b^2}p_b(1-p_b)=-2<0 .
\end{equation}
Thus the unique maximizer is $p_b=1/2$. The variance is zero at $p_b\in\{0,1\}$, corresponding to rollout groups that are either all failures or all successes.
\end{proof}

\textbf{Recoverability gain.}
Proposition~\ref{prop:midpoint_reward_variance} shows that binary reward contrast is maximized at the midpoint and vanishes at both saturated extremes. Therefore, rollouts with $p_b(z)\approx 0$ are effectively unrecoverable under intervention $b$, rollouts with $p_b(z)\approx 1$ are already solved, and rollouts near $p_b(z)=0.5$ provide the richest expected reward contrast. This motivates measuring recoverability by how much an intervention improves the reward contrast of the resulting rollout group.

In the population view above, this contrast is captured by the reward variance $V_b(z)=\operatorname{Var}(R_b(z))$. In the actual finite-group setting used by RAIL, we estimate the same signal through the empirical reward contrast
\begin{equation}
\label{eq:finite_group_contrast_app}
    I(Y)
    =
    \frac{1}{|Y|}
    \sum_{y_i\in Y}
    \left(R(y_i)-\bar R_Y\right)^2 .
\end{equation}
Thus, for a rollout state $z$ and intervention $b\in\mathcal{B}$, RAIL defines recoverability as the expected intervention-induced gain in finite-group reward contrast:
\begin{equation}
\label{eq:recoverability_gain_app}
    \Delta_\theta(z,b)
    =
    \mathbb{E}\!\left[
        I(Y_b)-I(Y_{\varnothing})
        \mid z,b,\pi_\theta
    \right],
\end{equation}
where $Y_b$ denotes the finite rollout group generated after applying intervention $b$, and $Y_{\varnothing}$ denotes the rollout group generated under default continuation. A positive gain $\Delta_\theta(z,b)>0$ indicates that the intervention exposes additional relative reward signal beyond default rollout generation. Conversely, $\Delta_\theta(z,b)\le 0$ indicates that the intervention does not improve the reward-distribution signal and should not be preferred unless compensated by another objective, such as exploration diversity or downstream verifier information.

\textbf{Takeaway.}
The midpoint target is not an arbitrary heuristic. For binary verifiable rewards, $p=0.5$ is the unique point at which the rollout distribution maximizes learning signals. Since group-based policy optimization relies on relative reward differences, these reward contrast provides a principled proxy for the useful reward-dependent signal available to the update. RAIL therefore instantiates recoverability as the intervention-induced improvement in finite-group reward contrast:
\[
    \Delta_\theta(z,b)
    =
    \mathbb{E}\!\left[
        I(Y_b)-I(Y_{\varnothing})
        \mid z,b,\pi_\theta
    \right].
\]
This quantity measures whether an intervention makes a task or trajectory state more informative for policy optimization, rather than merely increasing the number of generated rollouts.

\subsection{Recoverability Regret Analysis}
\label{app:regret}

This appendix formalizes the recoverability-regret guarantee for RAIL. Rather than proving learnability of the recoverability function from first principles \citep{}, our goal is to show that once recoverability can be tracked with sublinear cumulative error, the resulting intervention controller avoids the linear recoverability gap incurred by fixed or restricted heuristics. The proof uses a standard bandit decomposition: action selection regret can be controlled by the discrepancy between the estimated and true action values, together with an exploration term. This follows the finite-time regret perspective of classical finite-action bandits~\citep{auer2002finite} and the importance-weighted partial-feedback view of adversarial bandits~\citep{auer2002nonstochastic}. Because RAIL observes side information in the form of trajectory states, we use the oracle-style contextual-bandit perspective, where supervised or cost-sensitive prediction oracles are used to estimate action values under bandit feedback~\citep{langford2008epoch, agarwal2014taming}. Finally, because the recoverability function changes as $\pi_{\theta_t}$ evolves, we model this effect through a variation budget, following classical non-stationary bandit analyses based on switching or drifting reward distributions~\citep{garivier2011upper, besbes2014stochastic, besbes2015nonstationary}. VIP~\citep{nguyen2026adaptive} provides a closely related variance-informed rollout-allocation theory at the prompt level; RAIL instead studies online regret over a structured state-level intervention space.

\textbf{Setup.}
Let $\mathcal{B}$ be a finite intervention space with $K=|\mathcal{B}|$. At training step $t$, the current policy $\pi_{\theta_t}$ induces a recoverability-gain function
\begin{equation}
\label{eq:app_recoverability_function}
    f_t(z,b)
    =
    \Delta_{\theta_t}(z,b)
    =
    V_{\theta_t,b}(z)-V_{\theta_t,\varnothing}(z).
\end{equation}
Here $z\sim\mathcal{D}_t$ is the rollout state encountered during training and $b\in\mathcal{B}$ is an intervention. The controller selects $b_t$ and observes noisy bandit feedback
\begin{equation}
\label{eq:app_bandit_observation}
    y_t=f_t(z_t,b_t)+\xi_t,
    \qquad
    \mathbb{E}[\xi_t\mid z_t,b_t,\mathcal{H}_{t-1}]=0,
    \qquad
    |\xi_t|\le \sigma ,
\end{equation}
where $\mathcal{H}_{t-1}$ is the interaction history before step $t$. The cumulative recoverability regret is
\begin{equation}
\label{eq:app_recoverability_regret}
    \mathrm{Reg}_T
    =
    \sum_{t=1}^{T}
    \mathbb{E}_{z_t\sim\mathcal{D}_t}
    \left[
        \max_{b\in\mathcal{B}} f_t(z_t,b)
        -
        f_t(z_t,b_t)
    \right].
\end{equation}
This regret compares the selected intervention against the best intervention in the structured intervention space at each step.

\textbf{Assumptions.}
We use five assumptions. The first two are standard finite-action contextual-bandit conditions: bounded feedback and sufficient exploration. The next two isolate the main difficulties of our setting: learnability of recoverability and non-stationarity induced by policy optimization. The final assumption is an oracle-style tracking condition that combines these quantities into a cumulative estimation-error bound for the recoverability controller.

\begin{assumption}[Bounded finite-action feedback]
\label{assump:bounded_finite}
The intervention space $\mathcal{B}$ is finite. and $f_t(z,b)\in[0,1]$ for all $t,z,b$. The feedback noise is conditionally mean-zero and bounded as in~\eqref{eq:app_bandit_observation}.
\end{assumption}

\begin{assumption}[Sufficient exploration]
\label{assump:exploration}
The controller uses a greedy-with-exploration policy: with probability $1-\rho_t$ it selects the intervention maximizing the current estimate, and with probability $\rho_t$ it explores over $\mathcal{B}$. The exploration schedule satisfies
\begin{equation}
\label{eq:min_exploration}
    \Pr(b_t=b\mid z_t,\mathcal{H}_{t-1})\ge \rho_t/K,
    \qquad
    \forall b\in\mathcal{B},
    \qquad
    \sum_{t=1}^{T}\rho_t=\tilde O(\sqrt{TK}).
\end{equation}
This condition ensures nonzero coverage of every intervention. It is the same role played by explicit exploration in finite-action bandits~\citep{auer2002finite} and by importance-weighted exploration in partial-feedback bandit algorithms~\citep{auer2002nonstochastic}.
\end{assumption}

\textbf{How RAIL enforces sufficient exploration.}
Assumption~\ref{assump:exploration} is implemented at the level of the finite intervention space, not at the level of primitive language actions. During live training, RAIL uses a greedy-with-exploration policy over the nine executable branch actions in $\mathcal{B}$. With probability $1-\rho_t$, the controller selects the branch action with the largest predicted recoverability gain if over the live utility gate; with probability $\rho_t$, it executes an exploratory branch action. This ensures that intervention learning is not driven only by the controller's current argmax, which would otherwise create self-reinforcing blind spots in rarely selected intervention cells.

In implementation, exploration is coverage-directed rather than uniformly random. When the exploration draw fires, RAIL selects the least-covered branch cell among the $|\mathcal{M}\times\mathcal{T}|=9$ executable interventions, with ties broken deterministically. Thus every intervention cell receives explicit coverage during live training as the controller becomes informative. This exploration probability is decayed over training, so exploration supplies the coverage needed for partial-feedback learning without asymptotically dominating the rollout budget. This realizes the role of $\rho_t$ in Assumption~\ref{assump:exploration}: each branch action receives nonzero exploration support, and the cumulative amount of exploration remains controlled. In this sense, sufficient exploration is guaranteed for learning relative utilities among executable interventions, while the act/skip decision remains governed by the utility gate.

\begin{assumption}[Recoverability learnability]
\label{assump:approximation}
Let $\mathcal{F}$ be the controller class used to estimate recoverability gains. Define the cumulative best-in-class approximation error
\begin{equation}
\label{eq:approx_error}
    A_T
    =
    \sum_{t=1}^{T}
    \inf_{g\in\mathcal{F}}
    \mathbb{E}_{z\sim\mathcal{D}_t}
    \left[
        \max_{b\in\mathcal{B}}
        |g(z,b)-f_t(z,b)|
    \right].
\end{equation}
We assume $A_T=o(T)$ in regimes where intervention utility is learnable from outcome traces.
\end{assumption}
Note this assumption is following the standard practice. We further show that this assumption is also widely supported across tasks via the adaptability experiments we demonstrate in Figure~\ref{fig:adapt}.

\begin{assumption}[Controlled policy-induced drift]
\label{assump:drift}
The cumulative variation of the recoverability function is bounded by
\begin{equation}
\label{eq:drift_budget}
    D_T
    =
    \sum_{t=2}^{T}
    \sup_{z\in\mathcal{Z},\,b\in\mathcal{B}}
    |f_t(z,b)-f_{t-1}(z,b)|.
\end{equation}
We assume $D_T=o(T)$ under stable policy optimization.
\end{assumption}

Assumption~\ref{assump:drift} is the stability condition connecting the controller analysis to policy optimization. If policy updates changed recoverability arbitrarily by a constant amount at every step, no online controller could track intervention utility, and the post-training process itself would be unstable. Thus, sublinear drift is a standard tracking premise of policy optimization rather than an additional claim specific to RAIL. Similar variation-budget conditions are commonly used to formalize learnability under changing reward distributions in non-stationary bandits~\citep{besbes2014stochastic,besbes2015nonstationary}; switching-window analyses provide a related view when the environment changes in segments~\citep{garivier2011upper}.

This drift perspective also motivates the sliding-window and recency-weighted training. A trace collected at time $j<t$ provides information about $f_j$, whereas the controller at time $t$ must predict $f_t$. The mismatch between these two recoverability functions is controlled by the accumulated drift:
\begin{equation}
\label{eq:drift_trace_bound}
    \sup_{z,b}|f_t(z,b)-f_j(z,b)|
    \le
    \sum_{s=j+1}^{t}
    \sup_{z,b}|f_s(z,b)-f_{s-1}(z,b)|.
\end{equation}
Thus, a larger window provides more intervention traces and reduces finite-sample estimation noise, but may introduce stale-data bias when the policy changes. Conversely, a smaller window better tracks the current recoverability function, but relies on fewer samples and can yield noisier estimates. The window size and recency weights therefore balance finite-sample estimation error against policy-induced drift.

\begin{assumption}[Online recoverability tracking]
\label{assump:tracking_oracle}
The online regression procedure used by the recoverability controller produces estimates $\widehat f_t$ satisfying the cumulative tracking bound
\begin{equation}
\label{eq:tracking_oracle_bound}
    \epsilon_T
    :=
    \sum_{t=1}^{T}
    \mathbb{E}_{z_t}
    \left[
        \max_{b\in\mathcal{B}}
        |\widehat f_t(z_t,b)-f_t(z_t,b)|
    \right]
    \le
    \tilde O(\sqrt{TK}) + O(D_T) + O(A_T).
\end{equation}
\end{assumption}

Assumption~\ref{assump:tracking_oracle} is the main online tracking condition in the theorem. It states that the recoverability controller can follow the evolving intervention-utility function with controlled cumulative prediction error. It compactly captures the contextual-bandit reduction: under sufficient exploration, bandit feedback can be converted into estimable action-value information, and regret can be bounded through the quality of the induced prediction problem. This view is standard in contextual bandits with side information~\citep{langford2008epoch} and is developed more explicitly through cost-sensitive classification oracle reductions~\citep{agarwal2014taming}. The variation-budget literature motivates modeling policy-induced non-stationarity through the drift term $D_T$~\citep{besbes2014stochastic,besbes2015nonstationary}. In our analysis, the additive contribution of $D_T$ enters through the tracking condition in Assumption~\ref{assump:tracking_oracle}. Theorem~\ref{thm:rail_regret_oracle} then converts this tracking guarantee into an action-selection regret bound.

\textbf{Main regret theorem.}
Let $b_t^\star(z)=\arg\max_{b\in\mathcal{B}} f_t(z,b)$
be the best intervention at state $z$ under the current policy. RAIL selects interventions greedily with respect to $\widehat f_t$ except for the exploration mass in Assumption~\ref{assump:exploration}.

\begin{theorem}[Recoverability regret under online tracking]
\label{thm:rail_regret_oracle}
Under Assumptions~\ref{assump:bounded_finite}--\ref{assump:tracking_oracle}, the recoverability controller satisfies
\begin{equation}
\label{eq:rail_regret_tracking}
    \mathrm{Reg}_T
    \le
    2\epsilon_T+\sum_{t=1}^{T}\rho_t .
\end{equation}
In particular, if
\begin{equation}
\label{eq:epsilon_instantiation}
    \epsilon_T
    =
    \tilde O(\sqrt{TK})+O(D_T)+O(A_T),
    \qquad
    \sum_{t=1}^{T}\rho_t=\tilde O(\sqrt{TK}),
\end{equation}
then
\begin{equation}
\label{eq:rail_regret_bound}
    \mathrm{Reg}_T
    \le
    \tilde O(\sqrt{TK})+O(D_T)+O(A_T),
\end{equation}
where constant factors are absorbed into the $O(\cdot)$ and $\widetilde O(\cdot)$ notation. Dividing ~\eqref{eq:rail_regret_bound} by $T$ gives
\begin{equation}
\label{eq:average_regret_decomposition}
\frac{\operatorname{Reg}_T}{T}
\le
\frac{\widetilde O(\sqrt{TK})}{T}
+
\frac{O(D_T)}{T}
+
\frac{O(A_T)}{T}.
\end{equation}
Since the intervention space is finite and $K=|\mathcal B|$ is constant, the finite-action statistical and exploration term is $o(1)$. Moreover, if $D_T+A_T=o(T)$ with $D_T,A_T\ge 0$, then $D_T/T=o(1)$ and $A_T/T=o(1)$, so the drift and approximation terms also vanish. Therefore,
\begin{equation}
\label{eq:average_regret_vanishes}
    \frac{\mathrm{Reg}_T}{T}\to 0 .
\end{equation}
This shows that the controller may accumulate regret during training, but its average recoverability regret vanishes under sublinear recoverability drift and sublinear cumulative approximation error.
\end{theorem}

\begin{proof}
First consider the greedy part of the controller. Let
$b_t=\arg\max_{b\in\mathcal{B}}\widehat f_t(z_t,b)$. For any $z_t$, we have
\begin{equation}
\label{eq:greedy_regret_decomposition}
\begin{aligned}
    f_t(z_t,b_t^\star)-f_t(z_t,b_t)
    &=
    f_t(z_t,b_t^\star)-\widehat f_t(z_t,b_t^\star)
    +\widehat f_t(z_t,b_t^\star)-\widehat f_t(z_t,b_t)  \\
    &\qquad
    +\widehat f_t(z_t,b_t)-f_t(z_t,b_t)  \\
    &\le
    2\max_{b\in\mathcal{B}}
    |\widehat f_t(z_t,b)-f_t(z_t,b)| .
\end{aligned}
\end{equation}
The first term is the estimation error at the truly optimal intervention $b_t^\star$; it can increase the bound when $b_t^\star$ is underestimated. The middle term is non-positive because $b_t\in\arg\max_{b\in\mathcal B}\widehat f_t(z_t,b)$, so $\widehat f_t(z_t,b_t)\geq\widehat f_t(z_t,b_t^\star)$. The final term is the estimation error at the selected intervention $b_t$; it can increase the bound when $b_t$ is overestimated. Thus, regret can arise from underestimating the optimal intervention or overestimating the selected intervention, and both effects are bounded by $\max_{b\in\mathcal B}|\widehat f_t(z_t,b)-f_t(z_t,b)|$. This gives the factor of two in~\eqref{eq:greedy_regret_decomposition}.

Exploration contributes at most $\rho_t$ regret at step $t$ because recoverability gains are bounded in $[0,1]$. Therefore,
\begin{equation}
\label{eq:regret_sum_bound}
\begin{aligned}
    \mathrm{Reg}_T
    &\le
    2\sum_{t=1}^{T}
    \mathbb{E}_{z_t}
    \left[
        \max_{b\in\mathcal{B}}
        |\widehat f_t(z_t,b)-f_t(z_t,b)|
    \right]
    +
    \sum_{t=1}^{T}\rho_t  \\
    &=
    2\epsilon_T+\sum_{t=1}^{T}\rho_t .
\end{aligned}
\end{equation}
This proves~\eqref{eq:rail_regret_tracking}. Substituting~\eqref{eq:epsilon_instantiation} gives~\eqref{eq:rail_regret_bound}. Finally,~\eqref{eq:average_regret_decomposition} shows that the average regret vanishes when $K$ is finite and $D_T+A_T=o(T)$, proving~\eqref{eq:average_regret_vanishes}.
\end{proof}

\textbf{Why fixed or restricted intervention rules can incur linear regret.}
The theorem above establishes that RAIL can track the best structured intervention when recoverability is learnable and policy-induced drift is controlled. We now show why fixed heuristics or restricted action spaces can still incur linear recoverability regret.

Let $\pi^{\mathrm{base}}$ be a baseline intervention policy. If $\pi^{\mathrm{base}}$ is randomized, $f_t(z_t,\pi^{\mathrm{base}}(z_t))$ denotes the expectation over the baseline's action distribution. Define its per-step recoverability gap as
\begin{equation}
\label{eq:baseline_gap}
    \gamma_t(\pi^{\mathrm{base}})
    =
    \mathbb{E}_{z_t}
    \left[
        \max_{b\in\mathcal{B}} f_t(z_t,b)
        -
        f_t(z_t,\pi^{\mathrm{base}}(z_t))
    \right].
\end{equation}

\begin{proposition}[Persistent mismatch implies linear recoverability regret]
\label{prop:linear_baseline_regret}
If there exists a constant $\gamma>0$ and a set of steps
$\mathcal{T}\subseteq\{1,\ldots,T\}$ with $|\mathcal{T}|=\Omega(T)$ such that
\begin{equation}
\label{eq:persistent_gap_condition}
    \gamma_t(\pi^{\mathrm{base}})\ge \gamma,
    \qquad
    \forall t\in\mathcal{T},
\end{equation}
then the baseline incurs linear recoverability regret:
\begin{equation}
\label{eq:linear_baseline_regret}
    \mathrm{Reg}_T(\pi^{\mathrm{base}})
    =
    \sum_{t=1}^{T}\gamma_t(\pi^{\mathrm{base}})
    \ge
    \gamma|\mathcal{T}|
    =
    \Omega(T).
\end{equation}
\end{proposition}

\begin{proof}
The result follows by summing over the subset $\mathcal{T}$:
\begin{equation}
\label{eq:linear_baseline_proof}
    \mathrm{Reg}_T(\pi^{\mathrm{base}})
    =
    \sum_{t=1}^{T}\gamma_t(\pi^{\mathrm{base}})
    \ge
    \sum_{t\in\mathcal{T}}\gamma_t(\pi^{\mathrm{base}})
    \ge
    \gamma|\mathcal{T}|
    =
    \Omega(T).
\end{equation}
\end{proof}

\textbf{Application to representative baselines.}
Proposition~\ref{prop:linear_baseline_regret} is a conditional lower bound: it does not claim that every baseline always fails, but identifies the condition under which a baseline accumulates linear recoverability regret.

\begin{itemize}
    \item \textbf{Fixed uncertainty-triggered branching, e.g., ARPO-style methods.}
    Such methods use a fixed proxy, such as entropy or uncertainty, to decide when to branch. If this proxy remains persistently misaligned with recoverability. For example, high-uncertainty states are unrecoverable, or low-uncertainty states require recovery intervention. Then the method has a constant per-step recoverability gap on those states. Proposition~\ref{prop:linear_baseline_regret} then gives linear regret.

    \item \textbf{Prompt-level scalar allocation, e.g., VIP-style methods.}
    VIP is principled for prompt-level rollout allocation and provides important theoretical motivation for variance-informed budget control~\citep{nguyen2026adaptive}. However, relative to RAIL's structured comparator, VIP selects \emph{how many} rollouts to allocate to a prompt, not \emph{where} or \emph{how} to intervene inside a trajectory. If the best recoverability gain comes from structured choices outside the scalar allocation space, such as branching at a specific trajectory state, changing decoding behavior, or applying a recovery intervention, then scalar allocation has an irreducible representation gap. If this gap persists on a linear number of states, Proposition~\ref{prop:linear_baseline_regret} implies linear regret with respect to the structured-intervention comparator.

    \item \textbf{Static or proxy-driven intervention policies, e.g., AEPO-style methods.}
    More generally, any method that maps proxy signals to intervention decisions without updating from realized recoverability outcomes is vulnerable to persistent proxy mismatch. If the proxy-to-gain relation shifts as $\pi_{\theta_t}$ evolves, the method cannot correct its intervention policy from feedback, and its cumulative recoverability regret is linear whenever the mismatch persists.
\end{itemize}

\textbf{Takeaway.}
The regret analysis separates three sources of difficulty: statistical exploration over the finite intervention space, approximation error of the recoverability controller, and policy-induced drift. RAIL achieves sublinear recoverability regret when the structured intervention utility is learnable and changes slowly enough to be tracked. Fixed heuristics and scalar-only allocation methods can incur linear regret whenever their proxy or action-space restriction creates a persistent gap to the best structured intervention. This formalizes the two core gaps in the main text: the non-stationary gap appears through $D_T$, and the non-scalar gap appears through irreducible representation mismatch.

\section{Discussions}
\subsection{Intervention-Shaped Sampling and Policy Updates}
\label{app:policy_loss_distribution}

RAIL changes the distribution from which training trajectories are collected. Given an intervention $b$, a continuation is sampled from an intervention-shaped behavior distribution
\[
    y \sim Q_{\theta_{\mathrm{old}},b}(\cdot\mid z),
\]
where $Q_{\theta_{\mathrm{old}},b}$ is induced by the current policy together with the selected branch point, branch budget, and decoding regime, including its temperature and top-$p$ configuration. 
%The policy update, however, does not treat $Q_{\theta_{\mathrm{old}},b}$ as an explicit behavior policy. Generation-time log-probabilities are not used in the loss, and importance-sampling correction for the intervention-shaped decoder is disabled. Instead, the trainer recomputes token log-probabilities under the training policy and applies the standard group-relative update:
The policy update still uses the standard GRPO policy ratio between the current policy and the detached pre-update policy. However, it does not use
\(Q_{\theta_{\mathrm{old}},b}\) as the explicit behavior policy in the
denominator. In other words, RAIL does not add an additional
importance-correction term for the intervention-shaped decoder beyond the
standard GRPO ratio. Instead, the trainer recomputes token log-probabilities
under the training policy and applies the standard group-relative update:
$\widehat{\mathcal{L}}_{\mathrm{RAIL}}(\theta)
    =
    -\frac{1}{|\mathcal{Y}|}
    \sum_{y_i\in\mathcal{Y}}
    \widehat{A}_i
    \sum_{\tau}
    \log \pi_{\theta}
    \!\left(
        y_{i,\tau}
        \mid z_i,y_{i,<\tau}
    \right)$,
up to the clipping and masking operations of the underlying GRPO implementation. With the single inner optimization iteration used in our experiments, the reference log-probabilities are detached copies of the current training-policy log-probabilities, so the corresponding policy ratio is one at the update. Thus, branch trajectories generated with different temperatures or top-$p$ values enter the policy loss without temperature reweighting or a denominator based on $Q_{\theta_{\mathrm{old}},b}$.

Accordingly, RAIL should be understood as \emph{intervention-shaped data collection with an approximate on-policy update}. It is on-policy in the limited sense that trajectories are produced by the current policy parameters and are consumed immediately, rather than replayed from an old policy checkpoint. It is not exactly on-policy with respect to the token-level sampling distribution, because the intervention changes that distribution while the loss is evaluated using $\pi_\theta$. Likewise, our statement that RAIL leaves the optimization objective unchanged refers to the algebraic form of the critic-free policy loss and advantage estimator; RAIL does change the empirical rollout distribution over which that loss is evaluated. One could instead define a corrected update using
\[
    \frac{
        \pi_{\theta}(y\mid z)
    }{
        Q_{\theta_{\mathrm{old}},b}(y\mid z)
    },
\]
or its token-level analogue. Such a correction would require tracking the complete intervention-induced behavior probability, including temperature scaling, top-$p$ renormalization, anchor selection, and branching decisions. It would also require suitable support conditions: truncated decoding distributions such as top-$p$ sampling do not generally provide full support for the original policy distribution. Sequence-level ratios can consequently be unavailable or high-variance, potentially undermining the stable group-relative update that RAIL is designed to preserve.

We therefore adopt the uncorrected formulation deliberately. It isolates RAIL's contribution to the rollout-generation layer, remains compatible with standard critic-free optimization implementations, and avoids introducing a separate high-variance off-policy correction problem. The resulting update should not be interpreted as an unbiased estimator of the original-policy expectation under arbitrary interventions; rather, it optimizes the standard GRPO loss on a selectively reshaped set of trajectories. Developing behavior-corrected intervention learning with explicit support control is a complementary direction beyond the scope of this work.

\subsection{Recency-Weighted Tracking under Policy Drift}
\label{app:window_tradeoff}

The recoverability controller learns from historical intervention traces, while its target evolves with the policy. A trace collected at step $j<t$ reflects
$f_j(z,b)=\Delta_{\theta_j}(z,b)$, whereas the controller at step $t$ must predict $f_t$. Retaining more traces reduces statistical estimation error, but older traces may become stale. This is the standard remembering--forgetting trade-off in non-stationary online learning: longer histories improve statistical precision, whereas shorter histories improve responsiveness to change~\citep{besbes2014stochastic}. We summarize policy-induced non-stationarity by
\[
    \delta_t
    =
    \sup_{z,b}
    \left|
        f_t(z,b)-f_{t-1}(z,b)
    \right|,
    \qquad
    D_T=\sum_{t=2}^{T}\delta_t .
\]
This variation measure allows recoverability to change smoothly, abruptly, or non-uniformly over training. Its role is to capture the total amount of decision-relevant change, rather than impose a specific temporal model. The additive $O(D_T)$ term in our tracking condition follows from bounded-memory estimation. A trace collected at step $j$ differs from the current target by at most
\[
    \sup_{z,b}|f_t(z,b)-f_j(z,b)|
    \le
    \sum_{s=j+1}^{t}\delta_s .
\]
Under a fixed sliding window or exponential recency weighting with bounded effective memory, each local change $\delta_s$ influences only a bounded amount of future training weight. Consequently, the cumulative stale-target error is bounded by a constant multiple of
$\sum_s\delta_s=D_T$, and this constant is absorbed into the $O(D_T)$ term. The controller's cumulative tracking error can therefore be interpreted as the sum of statistical estimation error, approximation error, and drift-induced stale supervision.

We emphasize that the condition $D_T=o(T)$ expresses the stability regime required for online tracking. It does not require any particular drift pattern; it only requires that recoverability does not change by a persistent constant amount over a linear number of steps. If instead $D_T=\Omega(T)$, historical outcomes may become obsolete as quickly as they are collected, and vanishing average regret against a dynamic intervention oracle is generally impossible without stronger structure. \textbf{Nevertheless, such instability would also undermine the premise that successive policy updates can reuse experience from earlier policies, therefore the discussion of controller's regret becomes meaningless.} Accordingly, $O(D_T)$ should be understood as a sufficient tracking abstraction for RAIL's bounded-memory controller, not as a minimax characterization of arbitrary non-stationary contextual bandits. In practice, RAIL implements this abstraction through exponential recency weighting: recent traces receive greater weight, while older traces decay gradually, balancing statistical stability against adaptation to the evolving policy.

\subsection{Entropy-Guided Anchor Selection in Shadow Mode}
\label{app:anchor_selection}

Entropy is widely used for step-level exploration because it is available before outcome feedback and identifies decisions where the policy assigns substantial probability to competing continuations~\citep{zheng2025first, wei2026entropy, dong2026agentic, dong2025agentic}. We acknowledge the limitation that the reliance on a shadow warm-up that selects high-entropy anchors may bias the controller’s early supervision toward specific uncertainty patterns and could limit coverage of recoverable but low-entropy states. Naturally, low-entropy states may still be recoverable, for example when the policy is confidently wrong in rare cases; however, exhaustive search over all states would be computationally infeasible and substantially dilute the limited shadow budget. Under a limited shadow-mode budget, high-entropy states provide a practical high-yield proposal set: branching at locally uncertain decisions is more likely to generate diverse continuations than branching indiscriminately across the trajectory. Importantly, RAIL uses entropy only to \emph{propose} anchors, not to determine their recoverability. At each proposed anchor, the iterative procedure evaluates the realized gain of increasing branch budgets and stops once additional branching is no longer useful. Thus, intervention outcomes, rather than entropy itself, determine which trials are retained and provide both positive and negative supervision for the controller.  These trajectories are not discarded from policy optimization, but are simply not prioritized for additional branching during controller bootstrapping. The entropy filter therefore represents a deliberate coverage-efficiency trade-off, while iterative outcome-based validation limits dependence on the heuristic proposal signal.

\section{Prompt Design}
In this section, we provide the prompts for all four agentic benchmarks. We utilize a Reasoning-Then-Act template for the training. These components are essential for interpretability and reproducibility.

% =========================================================
% AgentBench OS
% =========================================================
\begin{figure*}[h]
\centering
\begin{promptbox}[AgentBench OS Prompt]{promptteal}

\promptsection{Role and Goal}
\begin{promptverb}
You are an AI assistant with full access to a Linux (Ubuntu) operating system through bash commands. You can execute bash commands to gather information, navigate directories, read files, and manipulate data. Your goal is to solve the given task using the provided tools.
\end{promptverb}

\promptsection{Available Tools}
\begin{promptverb}
- bash_action: Execute bash commands in the Linux environment. Use this tool to gather information and perform the requested operations.
- answer_action: Submit the final answer when the result is known.
\end{promptverb}

\promptsection{Interaction Protocol}
\begin{promptverb}
Before calling a tool, briefly state your reasoning in one sentence explaining what you plan to do and why. Then immediately call exactly ONE tool using <tool_call> tags.

Every turn must contain exactly one tool call. Use bash_action for commands such as grep, find, cat, awk, wc, and ls. Do not use interactive commands or commands requiring user input. When the answer is known, call answer_action with only the exact result value, such as "42" rather than "the answer is 42". If a command fails, use the error to choose a different approach.
\end{promptverb}

\promptsection{Examples}
\begin{promptverb}
User: How many .txt files are in /home?
Assistant: I need to find and count all .txt files in /home using find and wc.
<tool_call>
{"name":"bash_action","arguments":{"script":"find /home -name '*.txt' -type f | wc -l"}}
</tool_call>

Tool: 5
Assistant: The command returned 5, which is the answer.
<tool_call>
{"name":"answer_action","arguments":{"answer":"5"}}
</tool_call>

Always provide brief reasoning before calling a tool.
\end{promptverb}

\end{promptbox}
\vspace{-10pt}
\caption{System prompt used for \textbf{AgentBench OS}. Section labels are added only for presentation and organize the role, available tools, interaction protocol, and demonstrations.}
\label{fig:prompt_os}
\end{figure*}

% =========================================================
% AgentBench DB
% =========================================================
\begin{figure*}[t]
\centering
\begin{promptbox}[AgentBench DB Prompt]{promptblue}

\promptsection{Role and Goal}
\begin{promptverb}
You are an AI assistant with access to a MySQL database. You can execute SQL queries to retrieve data or modify the database. Your goal is to solve the given task using the provided tools.
\end{promptverb}

\promptsection{Available Tools}
\begin{promptverb}
- sql_query: Execute SELECT, INSERT, UPDATE, or DELETE queries.
- answer_action: Submit the final answer.
\end{promptverb}

\promptsection{Interaction Protocol}
\begin{promptverb}
Before calling a tool, briefly state your reasoning in one sentence explaining what you plan to do and why. Then immediately call exactly ONE tool using <tool_call> tags.

Always use backticks around table and column names. For SELECT tasks, call answer_action with only the query result value, such as "John Smith", "42", or "26.0". For INSERT, UPDATE, or DELETE tasks, first execute the write query using sql_query and then call answer_action with a short confirmation such as "done"; the grader evaluates the resulting database state rather than the answer text.

If a query fails, use the error to construct a corrected query and do not repeat the same query. If a successful query returns an empty result, verify the schema, table and column names, filter value, case sensitivity, exact-string matching, and data types. In particular, numerical comparisons over text columns may require CAST. Return "[]" or "No results" only when you are confident that the result is genuinely empty.
\end{promptverb}

\promptsection{Question Interpretation}
\begin{promptverb}
Some questions use automatically generated phrasing. "Name the X", "Tell me X", "Name the most X", and "Name the least X" ask for the VALUE of column X, which may be a number, date, or string, rather than a person's name. For example, "Name the most age" asks for the maximum age value. "How many X" expects only a count value.
\end{promptverb}

\promptsection{Examples}
\begin{promptverb}
User: What is the name of the person with ID 42 from the "User Data" table?
Assistant: I will retrieve the name associated with ID 42.
<tool_call>
{"name":"sql_query","arguments":{"query":"SELECT `name` FROM `User Data` WHERE `id` = 42"}}
</tool_call>

Tool: John Smith
Assistant: The query returned John Smith, so that is the answer.
<tool_call>
{"name":"answer_action","arguments":{"answer":"John Smith"}}
</tool_call>

Begin by understanding the task.
\end{promptverb}

\end{promptbox}
\vspace{-10pt}
\caption{System prompt used for \textbf{AgentBench DB}. Section labels are added only for presentation and organize the database interaction rules, answer protocol, question interpretation, and demonstrations.}
\label{fig:prompt_db}
\end{figure*}

% =========================================================
% WebShop
% =========================================================
\begin{figure*}[t]
\centering
\begin{promptbox}[WebShop Prompt]{promptorange}

\promptsection{Role and Goal}
\begin{promptverb}
You are a shopping agent in a text-based online store. Your task is to purchase the ONE product that best matches the user's instruction, including the requested product type, attributes, options, and price.
\end{promptverb}

\promptsection{Available Tools}
\begin{promptverb}
- search_action: Search the store using keywords when a search bar is available.
- click_action: Click an exact label shown in the current observation, including a product ID, option value, navigation control, information tab, or "Buy Now".
- answer_action: End the session when purchasing cannot be completed. Normally, prefer completing the task through "Buy Now".
\end{promptverb}

\promptsection{Store Interaction}
\begin{promptverb}
The current observation lists all clickable elements. Begin by searching with focused keywords derived from the instruction. Search results contain product IDs, titles, and prices. Click a promising product ID to open its page, inspect its description or features when necessary, select every requested option such as color or size, and click "Buy Now" to complete the purchase.

You are evaluated by how well the purchased product satisfies all requested attributes and options while remaining within the price limit. Refine the search when the results are poor, and click only labels that appear in the current observation.
\end{promptverb}

\promptsection{Interaction Protocol}
\begin{promptverb}
Every turn must contain one brief sentence of reasoning followed immediately by exactly ONE tool call using <tool_call> tags.
\end{promptverb}

\promptsection{Example}
\begin{promptverb}
Assistant: I will search for the requested product using its main attributes.
<tool_call>
{"name":"search_action","arguments":{"query":"machine washable decorative pillows 28 inch"}}
</tool_call>

Tool: [Results] B07XYZ1234  Machine Washable Decorative Throw Pillows.
Assistant: B07XYZ1234 matches the request, so I will open it.
<tool_call>
{"name":"click_action","arguments":{"button":"B07XYZ1234"}}
</tool_call>

Tool: [Product] options: color [blue] [gray]; size [28" x 28"] ... [Buy Now]
Assistant: I will select the requested size before purchasing.
<tool_call>
{"name":"click_action","arguments":{"button":"28\" x 28\""}}
</tool_call>

Tool: [Product] selected size: 28" x 28" ... [Buy Now]
Assistant: The product and selected options match the instruction, so I will complete the purchase.
<tool_call>
{"name":"click_action","arguments":{"button":"Buy Now"}}
</tool_call>


Begin.
\end{promptverb}

\end{promptbox}
\vspace{-10pt}
\caption{System prompt used for \textbf{WebShop}. Section labels are added only for presentation and organize the shopping objective, available actions, environment workflow, and tool-calling protocol.}
\label{fig:prompt_webshop}
\end{figure*}

% =========================================================
% ToolQA-Coffee
% =========================================================
\begin{figure*}[t]
\centering
\begin{promptbox}[ToolQA-Coffee Prompt]{promptviolet}

\promptsection{Role and Data}
\begin{promptverb}
You answer questions about coffee prices using a Python interpreter. A pandas DataFrame named df is preloaded with the columns Date, Open, High, Low, Close, and Volume. Date is a string in YYYY-MM-DD format; Open, High, Low, and Close are floating-point values; and Volume is an integer. You must query df to obtain the requested values rather than answering from memory.
\end{promptverb}

\promptsection{Available Tools}
\begin{promptverb}
- python_interpreter: Execute Python code with df and pandas available. Assign the computed result to a variable named ans.
- answer_action: Submit only the final answer value.
\end{promptverb}

\promptsection{Interaction Protocol}
\begin{promptverb}
On each turn, provide one sentence of reasoning followed immediately by exactly ONE tool call using <tool_call> tags. Use python_interpreter to query or compute over df and assign the result to ans. After observing the result, call answer_action with only the value, such as a number, "bullish", or "bearish".
\end{promptverb}

\promptsection{Tool-Call Formatting}
\begin{promptverb}
The <tool_call> body must be valid JSON. The value of "code" must be one double-quoted string. Use single quotes inside the Python code, encode internal line breaks as \n, and always close the code string with a final double quote before the closing braces.
\end{promptverb}

\promptsection{Example}
\begin{promptverb}
User: Was 2015-07-29 a bullish or bearish day for coffee?
Assistant: I will compare the closing and opening prices for that date.
<tool_call>
{"name":"python_interpreter","arguments":{"code":"row = df[df.Date == '2015-07-29']\nans = 'bullish' if row.Close.iloc[0] > row.Open.iloc[0] else 'bearish'\nprint(ans)"}}
</tool_call>

Tool: bearish
Assistant: The computed result is bearish.
<tool_call>
{"name":"answer_action","arguments":{"answer":"bearish"}}
</tool_call>

Begin.
\end{promptverb}

\end{promptbox}
\vspace{-10pt}
\caption{System prompt used for the \textbf{ToolQA-Coffee} benchmark. Section labels are added only for presentation and organize the preloaded data interface, available tools, execution protocol, JSON constraint, and demonstration.}
\label{fig:prompt_toolqa}
\end{figure*}

\end{document}